\documentclass{article}

\usepackage[preprint]{neurips_2026}

\usepackage[utf8]{inputenc} 
\usepackage[T1]{fontenc}    
\usepackage{colortbl}
\usepackage{hyperref}       
\usepackage{url}            
\usepackage{tabularx}
\usepackage{booktabs}       
\usepackage{amsfonts}       
\usepackage{nicefrac}       
\usepackage{microtype}      
\usepackage{longtable}
\usepackage{graphicx}
\usepackage{subcaption}

\usepackage[australian]{babel}
\usepackage{amssymb}
\usepackage{amsmath}
\usepackage{amsthm}
\usepackage{mathrsfs} 
\usepackage{textcomp} 

\usepackage{bm}
\usepackage{array}
\usepackage{booktabs}
\usepackage{mdwmath} 
\usepackage{mdwtab}
\usepackage[printonlyused,withpage]{acronym} 
\usepackage{mathabx} 
\usepackage{multirow} 
\usepackage[disable]{todonotes}
\usepackage{algorithm}
\usepackage{algorithmic}

\newtheorem{theorem}{Proposition}

\newtheorem{definition}{Definition}

\providecommand{\set}[1]{\mathcal{#1}} 
\makeatletter
\def\imod#1{\allowbreak\mkern10mu({\operator@font mod}\,#1)}
\makeatother

\makeatletter
\renewcommand*{\fps@figure}{htb}
\renewcommand*{\fps@table}{htb}
\makeatother
\newcommand{\Johnie}{Johnny}
\newcommand{\Neurocell}{Neurocell}

\title{Von Neumann Networks}

\author{Shekhar S. Chandra\\
	School of Electrical Engineering and Computer Science,\\
	The University of Queensland, Australia.\\
	\texttt{shekhar.chandra@uq.edu.au} }

\begin{document}

\maketitle


\begin{abstract}
 	In the mid-twentieth century, mathematician and polymath John von Neumann created a computational system on an array of cells as a simple model of the human brain, where each cell had one of a finite set of roles or states that he predicted would be modelled by a diffusion process. In this work, we show that such a system, when developed in a modern deep learning setting, enables the construction of an artificial neuron having specialized roles that can be learnt. We refer to this neuron as the Von Neumann neuron, and the resulting neural network from such neurons result in a self-engineered design whose architecture is only dependent on the structure and locations of its inputs and outputs on this cellular array. The mathematical framework for these \acp{VNN} is also constructed and shows that they are based on the extension of neural operators and the learning of Green's functions with convolutions on a cellular topology having a diffusion signature. We also prove that these \acp{VNN} are part of a more general computational system called \aclp{CM} that are computationally universal. Initial experiments show that \ac{VNN} based \aclp{MLP} outperform their equivalent deep learning variant on basic tasks, while being more parameter efficient and are capable of learning new types of tasks. This includes the ability to solve for and construct an extension of the Von Neumann (hardware) architecture common to all modern computers to cells and suggests new opportunities that could be explored.
\end{abstract}

\begin{figure}[h]
	\centering
	\includegraphics[width=0.75\textwidth]{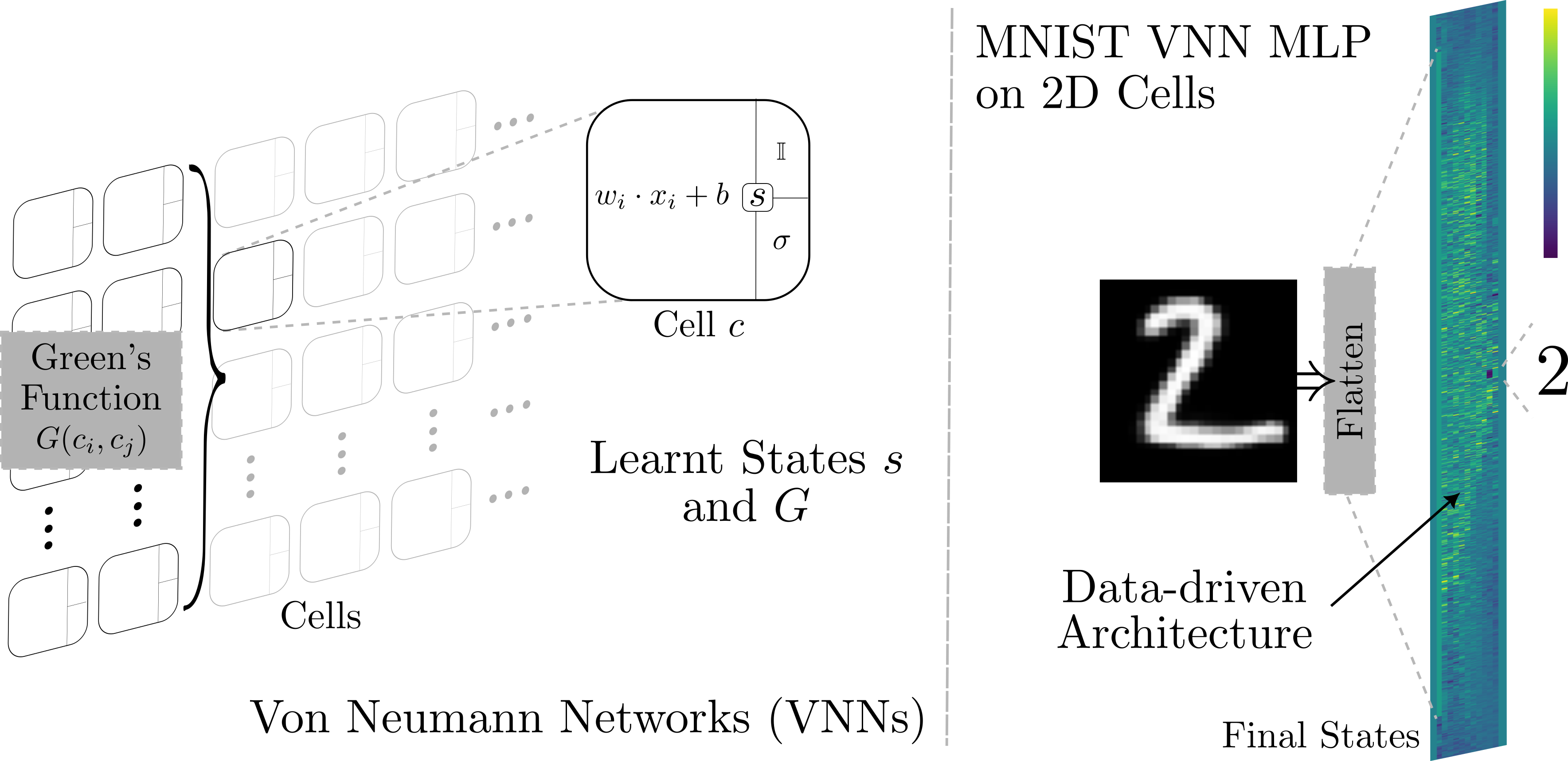}
	\caption{Inspired by the works of John von Neumann, we show that with a suitable state-based artificial neuron as a cell and neural operators implemented as convolutional Green's functions on cellular arrays, result in neural networks we call \aclp{VNN} capable of self-engineering its own connectivity and architecture entirely embedded within its array. The MNIST example above (right) shows the flattened input as a column input into a \ac{VNN} model results states representing cells that forms a data-driven encoder-like structure (yellow points) as network propagates to the output.}
	\label{fig::Overview}
\end{figure}

\acresetall
\section{Introduction}\label{sec::Introduction}
In the mid-twentieth century, John von Neumann (or \Johnie\footnote{We follow Freeman Dyson's convention~\citep[see Forward]{neumann_mathematical_2018}}~as he insisted~\citep{bhattacharya_man_2021}) envisioned a computational system on an infinite array of square cells to mimic the human brain. \Johnie~was motivated in understanding how the human brain functions and could be modelled as a digital machine~\citep[pg. 43]{neumann_computer_1958} and these simplified cellular structures were designed to approximate this model of a digital brain~\citep[pg. 93]{von_neumann_theory_1966}. His key insight in this model was that these cells could make decisions like neurons in the brain but would also maintain states depending on their purpose and then used it to construct a machine capable of reproducing itself~\citep[published posthumously]{von_neumann_theory_1966}. This theory of self-replicating machines is one of the founding works of \ac{CA}, in which simulations evolve the state of each cell in this array via a set of simple rules based on current states and their neighbours~\citep{sipper_fifty_1998}. He would eventually propose a continuous version of the model that would be ``a system of non-linear \acp{PDE}, essentially of the diffusion type''~\citep[pg. 95]{von_neumann_theory_1966} but he would not be able to write up this continuous model before his death in 1957.

In recent years, deep learning has resolved many problems associated with artificial neural networks and has been used to solve a large number of important problems~\citep{lecun_deep_2015} including more recently in language~\citep{vaswani_attention_2017}, chemistry~\citep{jumper_highly_2021} and computer vision~\citep{krizhevsky_imagenet_2012,he_deep_2016,liu_convnet_2022, rao_gfnet_2023}.

In what follows, we present a novel computational model on arrays of square cells that we prove to be computationally universal and result in a new neural network framework capable of learning their own architectures and connectivity (see Figure~\ref{fig::Overview}). This cellular model is based on a novel artificial neuron based on \Johnie's cell design that maintains its own role or state modernized for the deep learning era. The proposed \Neurocell~framework simulates the interaction of neurons with extensions to neural operators and learning neuron signal propagators known as Green's functions (equivalently \acp{PSF}) to cellular topologies with convolutional operators in order to propagate their impulses. The framework then enables each neuron's (continuous) state be optimized through stochastic gradient descent, thereby allowing cells to learn their own (continuous \ac{CA}-like) rules of evolution. The result is a discrete system that integrates a partial differential equation of diffusion type analogous to what \Johnie~predicted and models neural networks as structures embedded within these cells in what we will call \acp{VNN}. In the following sections, we review similar works to the proposed framework, its mathematical framework using Green's functions and results of experiments that show flexible networks capable of learning models of different shapes and tasks, including that of a model that can simulate an \ac{ALU} that is a core component of a modern computer's \ac{CPU}.


\section{Related Works}
\Johnie's primary motivation for a cellular system was to mimic the human brain with a simple computational model via cells representing neurons that had states enabling different roles. \Johnie~would eventually develop a cellular model that implements a machine with the ability to replicate itself by designing each cell in his system to have 29 different states for this replication process~\citep[pg. 132]{von_neumann_theory_1966}. ~\citet{codd_cellular_1968} would simplify the number of states from Johnie's 29 to just 8 including the important `signal' state that passes on signals within the array corresponding to the computation. Conway would show that a cellular system with significant complexity and capability of universal computation could be achieved with just two states, although self-replication has not been shown yet~\citep{gardner_mathematical_1970}. Since \Johnie's initial proposal of a 29 state self-replicating machine, majority of advancements have been largely focused on simplifying the number of states required, while creating various rule sets of \ac{CA} that produce Turing complete systems; see a brief review of this area provided in Appendix~\ref{sec::review-replication}.

After the development of traditional \ac{CA} approaches such as the \ac{GoL}~\citep[see Appendix~\ref{sec::GoL} for a brief introduction]{gardner_mathematical_1970, wolfram_statistical_1983, langton_self-reproduction_1984} that have hand-crafted rules, various attempts have been made to learn \ac{CA} rules.~\citet{chua_cellular_1988} introduced \acp{CeNN} that use hand-crafted templates and solutions to ordinary differential equations to create their networks prior to the deep learning era (see \citet{chua_cellular_1998} for an introduction). They would be capable of learning global structures for pattern creation as long as a suitable template for local connectivity was found. The templates usually have 19 parameters for a 3x3 neighborhood and would eventually be learnt through genetic algorithms and backpropagation algorithms from recurrent networks~\citep{chua_cellular_1998}. These \acp{CeNN} would later be extended to provide Turing completeness of their models~\citep{roska_cnn_1993}.

Neural operators were recently specifically introduced to learn the functions responsible for differentiating and integrating within networks~\citep{kovachki_neural_2023}. Additional work has shown that these operators can be computed efficiently using convolutions and Fourier space~\citep{li_fourier_2020} and recently showing that Green's functions themselves can be learnt~\citep{yoo_neural_2025}. Other works have attempted to improve neural networks by extending individual concepts critical in deep learning including the use of dynamic information routing~\citep{sabour_dynamic_2017} and learning of activation functions on network edges~\citep{liu_kan_2024}.

The recent work of~\citet{mordvintsev_growing_2020} has involved applying modern deep learning models to learn the \ac{CA} rules as a \ac{MLP} to generate specific patterns across its cells. Since it is impossible to know \emph{a priori} if a pattern produces complex behaviour without first simulating it in \acp{CA} like \ac{GoL}, generating patterns for applications is highly desirable. In their formulation however, each cell contains an \ac{MLP} that processes the neighbourhood and learns to produce a particular target pattern or image. These neural \ac{CA} together with morphogensis originally proposed by~\citet{turing_chemical_1990} are very well suited for generating textures on images~\citep{mordvintsev_nca_2021} and meshes~\citep{pajouheshgar_mesh_2024}. Recently, \citet{miotti_differentiable_2025} have even proposed using these neural \ac{CA} to represent logic systems. Recent work on differential logic gates has made advances in learning kernels using logic gates~\citep{petersen_deep_2022}. These logic gate networks have recently been combined with neural \ac{CA} to learn rules for \ac{CA} including the rules \ac{GoL}~\citep{miotti_differentiable_2025}.

Learning of architectures is also not new with the advent of \ac{NAS} in deep learning over the last decade especially when using reinforcement learning~\citep{zoph_neural_2017}. In general, these works optimize for the architecture either by using a more general pre-existing architecture such as trusses~\citep{liu_auto-deeplab_2019}, supernets~\citep{lu_neural_2021} or fabrics~\citep{saxena_convolutional_2016}) of inter-connected layers to select the optimal one or picking layers as needed via tree search~\citep{wang_sample-efficient_2021} (see~\citep{wang_advances_2024} for a detailed review).

Various attempts have been made to add memory and states to deep learning, but these have had largely unstable training, although they were able to learn to solve simple problems~\citep{graves_neural_2014, graves_hybrid_2016}. A few attempts have been made to produce learnable \ac{CA} including CoDI~\citep{gers_codi-1bit_1998} that relies on spiking networks and genetic encoding to simulate its network.

In this work, we show that our proposed \Neurocell~framework is a modern deep learning interpretation of \acp{CeNN} that is an extension of neural operators to a cellular array topology, in which learning Green's functions is simplified to regular convolution operations that known to be easier to optimize than differentiable logic gates (see Figure~\ref{fig::Overview}). The resulting \acp{VNN} will be shown to generalize deep learning as cellular machines capable of learning its own architecture in the following sections, or equivalently its own \ac{CA} rules and of universal computation with a direct mapping to a tape based self-engineered \ac{TM}. To better understand and precisely highlight the differences between our proposed VNNs and related works, we provide a comparison table in Appendix~\ref{sec::Comparison}. The VNNs are the only system that can leverage differentiable, continuous learning capabilities similar to Neural Operators and SSMs, yet apply them to the spatial, rule-based philosophy of Cellular Automata and CeNNs for self-organizing using efficient convolutions. Unlike NCAs that require an entire MLP network per cell to learn neighbourhood rules for interactions, VNNs show that its self-organization and healing only requires a single state parameter per cell and propagate signals via standard convolutions that are easier to learn than logic gates when compared to Differential Logic Gates. This allows it to elegantly self-engineer macroscopic architectures (like MLPs or ALUs) directly onto a cellular grid giving network weights spatial meaning.

\section{Methods}
In this section, the goal is to create a modern neural network framework, but on arrays of cells as John von Neuman originally intended. This involves creating a deep learning like neuron, but one that maintains its own state or role within the network called a Von Neumann Neuron. We extend this approach so that cells have learnable states and signal propagation letting them decide their role in the network (e.g. whether to activate or not etc.) and create a network with its own learnt architecture called \acp{VNN}. We then extend how signals propagate within the array allowing different neural operators to be learnt dictating how layers of cells can mix and connect in arbitrary dimensions. We will end with a description of the experiments used to show the properties of this framework. 

\subsection{Von Neumann Neuron}\label{sec::neuron}
The conventional deep learning artificial neuron has evolved from~\citet{mcculloch_logical_1943} style design shown in Figure~\ref{fig::nn}.
\begin{figure}[ht!]
	\centering
	\begin{subfigure}[b]{0.3\textwidth}
		\includegraphics[width=\textwidth]{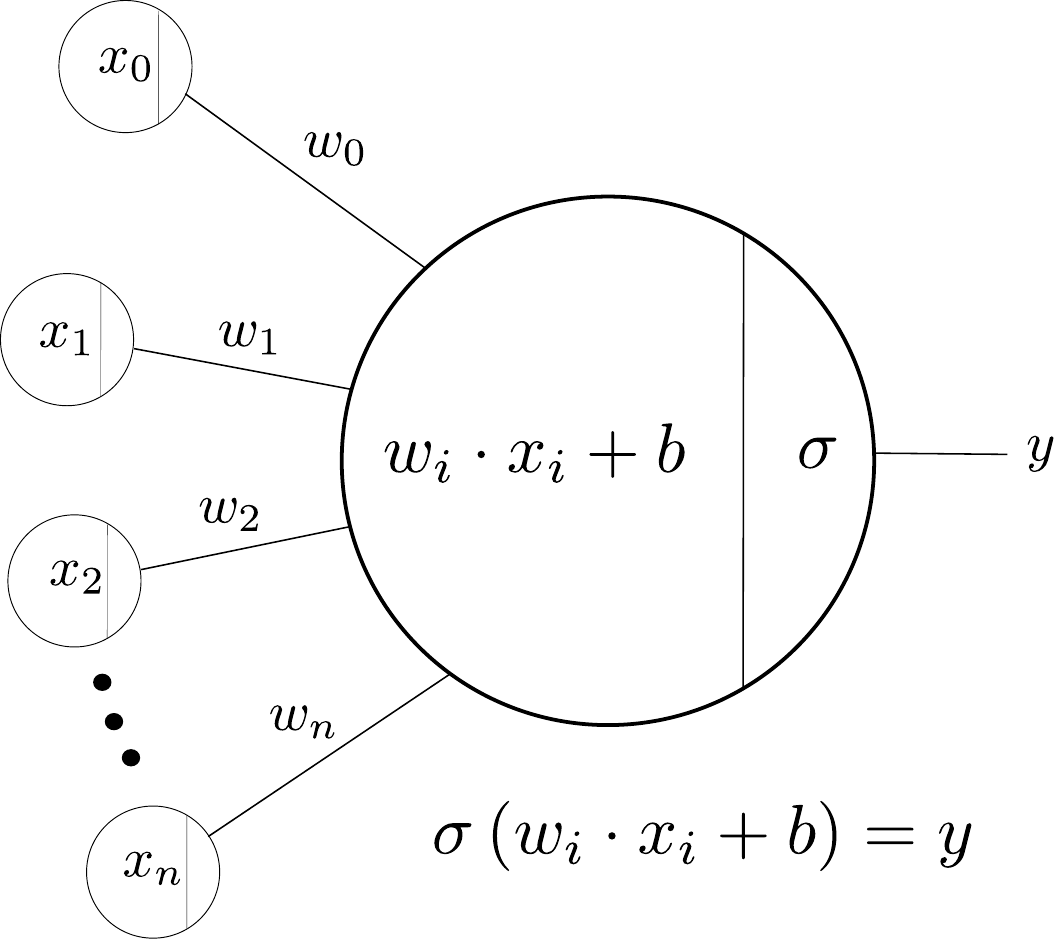}
		\caption{Deep learning neuron.}
		\label{fig::nn}
	\end{subfigure}
	\hfill 
	\begin{subfigure}[b]{0.3\textwidth}
		\includegraphics[width=\textwidth]{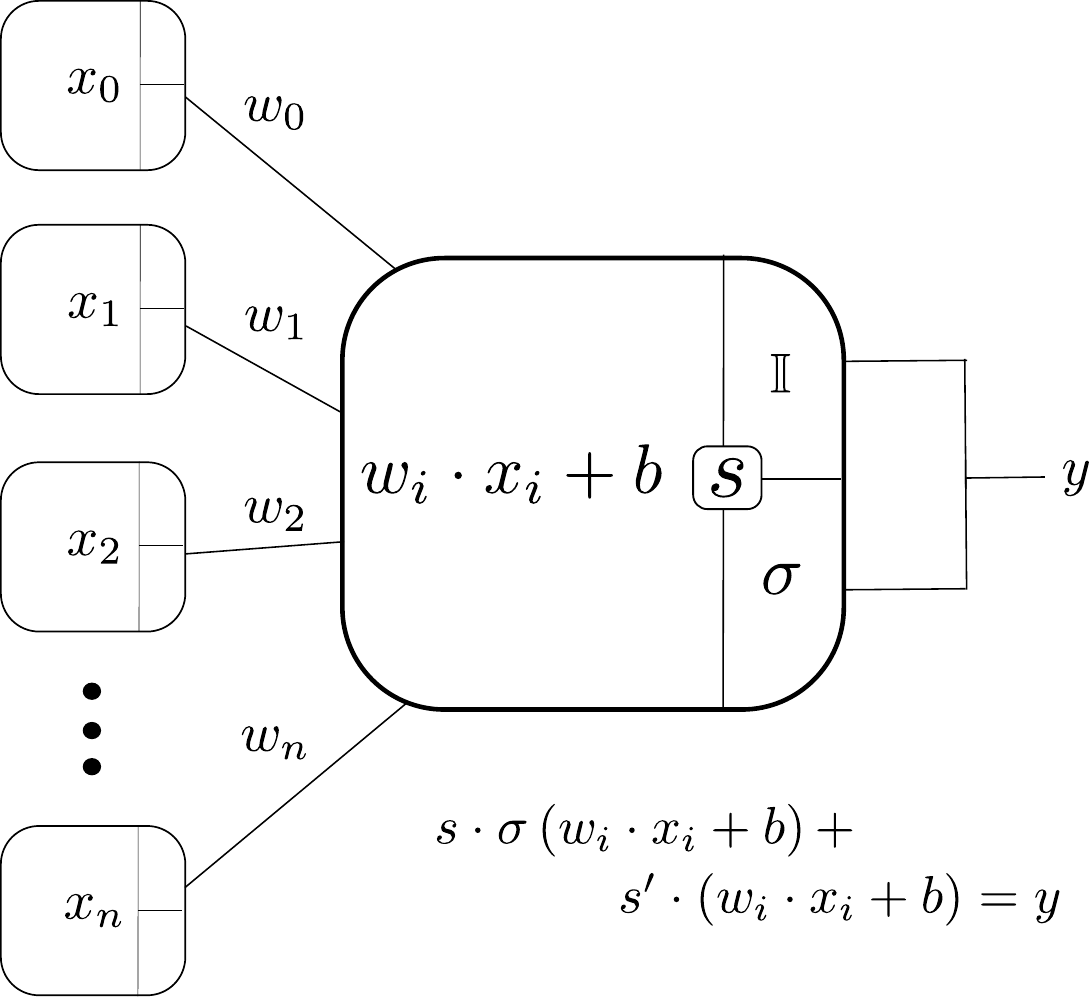}
		\caption{Proposed neuron.}
		\label{fig::vnn}
	\end{subfigure}
	\hfill 
	\begin{subfigure}[b]{0.35\textwidth}
		\includegraphics[width=\textwidth]{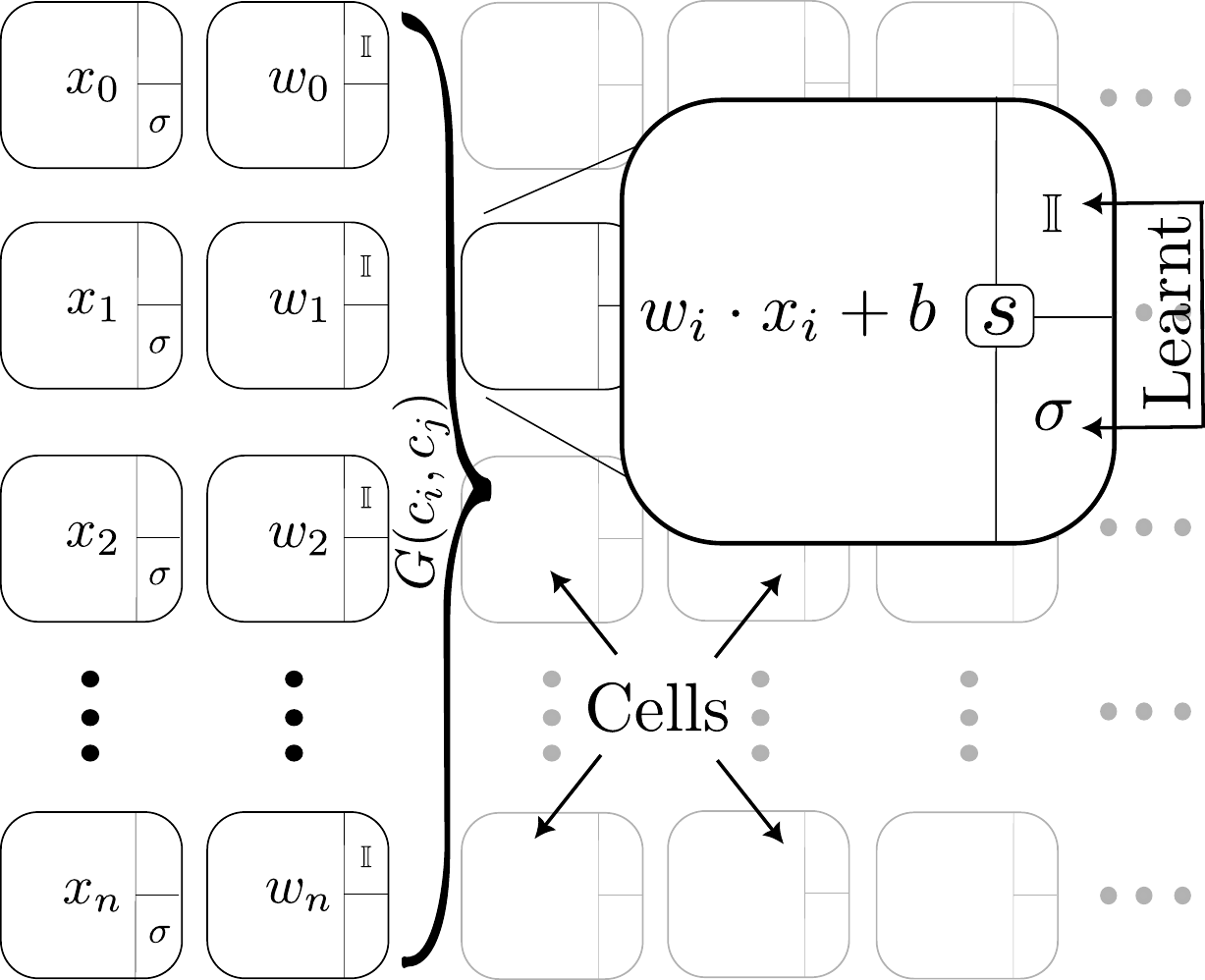}
		\caption{Proposed network on cells.}
		\label{fig::vnn_cells}
	\end{subfigure}
	\caption{The artificial neuron (a) represents the typical deep learning definition, where inputs $x_i$ are weighted by $w_i$ and biased with $b$ via (fixed) connections to this neuron then activated through the non-linear function $\sigma$ when firing, noting our use of the Einstein summation convention. (b) shows the proposed state-based Von Neumann neuron, where $s$ is the Codd state of the neuron and $s^\prime$ is the complement (usually $1-s$) enabling them to individually learn whether to passing on the signal or activate. (c) shows the proposed \ac{VNN} using these neurons embedded into a cellular topology, seamlessly allowing it to learn its own architecture suited to the objective function(s).}
	\label{fig::Neurons}
\end{figure}
In this approach, a neuron activates on a biased weighted sum of inputs to make a decision analogous to a non-linear switch. The network then involves a cascade of these non-linear decisions usually across fixed structures (such as multiple layers~\citep{amari_theory_1967}, trusses~\citep{liu_auto-deeplab_2019} or fabrics~\citep{saxena_convolutional_2016}) of inter-connected neurons that are functionally equivalent to a perceptron~\citep{rosenblatt_perceptron_1958}. The universal approximation theorem then ensures that any non-linear function can be approximated by the network to arbitrary accuracy as long as the structure has at least a single layer of infinite number of neurons~\citep{hornik_multilayer_1989, cybenko_approximation_1989} or many layers of finite number of neurons, i.e. it is deep enough\todo{Add Ref}.

This approach assumes that the neurons are correctly connected into the optimal architecture beforehand and that every neuron is required to partake decisions for each input of the network at every layer. Although some neurons can be designated to 'dropout' of the decision making in a process defined beforehand, the former does not take into account an evolving architecture based on the objective function and the latter makes sub-architectures, where neurons can specialize by choosing to pass on the decision to other neurons to make, very difficult to obtain unless they are pre-defined into the architecture. Both require a hand crafted or pre-defined fixed architectures to be in place during training. Even neural architecture search models, such as fabrics or trusses~\citep{saxena_convolutional_2016,liu_auto-deeplab_2019}, have a weighted form of fixed, pre-defined (albeit a very large) number of weighted sub-architectures.

We propose extending the usual definition of the deep learning neuron in Figure~\ref{fig::nn} to a Von Neumann neuron that also contains a state $s$ and use this state to enable neurons to toggle their role between passing on signals and traditional activation based decision making as shown in Figure~\ref{fig::vnn}. Depending on the value of this Codd state $s$, the neuron will pass on the biased weighted sum only (i.e. applying an Identity $\mathbb{I}$ operation instead of the activation $\sigma$) acting as a connection between neurons only (letting them defer decision making to other connected neurons) or make a non-linear decision with the designated activation function $\sigma$ as shown in Figures~\ref{fig::vnn} and~\ref{fig::vnn_cells}. We designate this Codd state $s$ as a learnable parameter and $s^\prime = 1-s$ taking inspiration from works utilizing contrastive loss\todo{Add Ref} and in the RealMLP~\citep{holzmuller_better_2024}.

The Von Neumann neuron will have two main advantages: Firstly, it enables the network weights $w_n$ to be directly embedded with the array of cells and act as connections or signal neurons etc. instead of only being capable of decision making. As we'll show in the subsequent section, we can either then hard-code layers ($s=1$) and connections ($s=0$) directly within the array of cells to mimic traditional deep learning architectures or learn the states $s$ directly through stochastic gradient descent. This will end up learning the network architecture with $s$ and $s^\prime$ rather than hard-coding them. 
The network then cascades not only non-linear decisions but linear superimposed signals across dynamic structures adjusted on the fly depending on the objective function. The weights of the network can also be observed directly distributed spatially across the array during its use. The result is the system's ability to learn its own network architecture and its many propagation functions or forms with Codd states $s$ via stochastic gradient descent.

Although our initial construction will be in view of neurons in arrays of cells in the subsequent section, this Von Neumann neuron can theoretically be used in any type of structured or unstructured (graphs etc.) topologies. In the next section, we describe how one can use this von Neumann neuron on array of cells to create Von Neumann Networks (VNNs) suited specifically for fast implementation on current array like digital hardware as \Johnie~intended with the help of current tensor based deep learning frameworks such as JAX\todo{Add Ref}.

\subsection{Von Neumann Networks}\label{sec::network}
\Johnie, after working with his close friend Stanis\l aw M. Ulam (creator of the Monte Carlo method), believed that the cellular model of a neural network based brain analogue would be ''more amendable to logical and mathematical treatment''~\citep[pg. 94]{von_neumann_theory_1966}. In this section, we propose such a construction through an extension of neural operators~\citep{kovachki_neural_2023} on a cellular topology and the recently introduced Neural Green's Functions~\citep{yoo_neural_2025} with our Von Neumann neurons of section~\ref{sec::neuron} in creating a neural network on cellular arrays we denote as \acp{VNN}.

\textbf{Approach.} With our Von Neumann neuron as a cell (see section~\ref{sec::neuron}), the cell's influence on its surroundings is initially equivalent to a unit response or Dirac $\delta$-function because it is isolated and unconnected. To build a network out of these cells, we need to expand the influence of each cell across the array keeping track of the cascading signals on a scalar field we will refer to as the Chua field $C$. This gradual smearing out of the neuron $\delta$-functions can be modelled by using the correct propagation function or~\ac{PSF} similar to Huygen's principle, where many point sources interact to create a wavefront, so that the \ac{PSF} then models the propagation of the cell signals. Since we eventually wish to utilize stochastic gradient descent to solve for the network connections via back propagation, the backward operator is a linear differential operator $\mathcal{D}$, this propagator will be the Green's function $G$ and the forward pass is the integral form of the differential operator involving this Green's function. We will show that this formulation is an extension of neural operators~\citep{kovachki_neural_2023} to a cellular topology, its implementation via convolutions~\citep{li_fourier_2020} and a cellular form of the recently introduced Neural Green's Functions~\citep{yoo_neural_2025}.

\textbf{Formulation.} Let each cell $c$ in the cellular array we call the Chua field $C$ be defined by the Von Neumann neuron of section~\ref{sec::neuron} and Figure~\ref{fig::vnn}. Given the universal approximation theorem, each cellular neuron $c$ contributes a small non-linear superposition into the final approximation by `divvying up' these contributions~\citep{hornik_multilayer_1989, cybenko_approximation_1989}. Following~\citep{kovachki_neural_2023}, let these contributions be represented the application of a linear differential operator $\mathcal{D}$ to our Chua field $C$ driven the by source $f(c)$ that represents the inputs as
\begin{equation}
	\mathcal{D} C(c) = f(c).\label{eqn::backward}
\end{equation}
Then the contribution or response of each cell with input cell $c_i$ at some other cell $c$ is given by $\delta(c-c_i)$ and depends on the response of the system called its Green's function $G$~\citep{arfken_mathematical_2001} as 
\begin{equation}
	\mathcal{D} G(c,c_i) = \delta(c-c_i). 
\end{equation}
The linearity of operator $\mathcal{D}$ allows the solution for any source $f(c)$ to be constructed via the total superposition of these contributions via the integral 
\begin{equation}
	C(c) = \int_{\Omega} G(c, c_i) f(c_i) dc,\label{eqn::forward}
\end{equation}
over the sub-field with boundary $\Omega$, thereby constructing the inverse operator $\mathcal{D}^{-1}$. Numerically, equation~\eqref{eqn::forward} is essentially the forward pass of the network given a suitable function $G$ and equation~\eqref{eqn::backward} is applied during back-propagation via automatic differentiation on the computational graph. It is important to note that in general the Green's function $G$ is independent of the sub-field with boundary $\Omega$ and the source term (and thus inputs) $f(c)$~\citep{yoo_neural_2025}. Thus, the final form of the field $C$ will contain the various input and output sources at the desired locations when optimizing the objective function and the regions in-between will form the ``layers'' of the \ac{VNN} according to the final Codd states $s$.

\textbf{Interpretation.} The Green's function the \ac{VNN} typically acts as a propagator or \ac{PSF} of the system that enhances or reduces the effects of the source $f(c)$ with respect to the different points in the field~\citep[pg. 550]{arfken_mathematical_2001}, and in our case it propagates the signals across the field $C$ to form connections and/or activations of the Von Neumann neurons comprising the cells $c$ according to the objective function and learnt Codd states $s$ (see Figure~\ref{fig::vnn_cells}).
\todo{quantum interp here}

\textbf{Implementation.} To implement the \ac{VNN}, we utilize a tensor form of equations~\eqref{eqn::backward}and~\eqref{eqn::forward}, the Chua field as a tensor $\bm{C}$ and the discrete forms of the operators $\bm{\mathcal{D}}$ \& $\bm{G}$. Since the \ac{VNN} is embedded within the cells, the Chua field $\bm{C}$ contains the input and outputs and are not restricted to any particular shape or number. As neural operators can be represented as convolution kernels~\citep{li_fourier_2020}, we are able to use convolutions in representing the Green's functions $G$ as large kernels and also learn multiple discrete Green's functions $\bm{G}_j$ for the system by adapting well known general convolution operators in various deep learning frameworks such as JAX~\citep{jax2018github}. The weights $w$, biases $b$ and Codd states $s$ are present in each of the cells, so that we may use the discrete fields for weights $\bm{W}$, biases $\bm{B}$ and states $\bm{S}$ respectively to represent the main constituents of the \ac{VNN}. Finally, the integration of the entire sub-field with boundary $\Omega$ will need to be broken down into multiple discrete steps that depend on the distance between the various inputs \& outputs, as well as the size of the kernel representing the Green's functions $\bm{G}_j$.

To compute a step of the forward pass of the randomly initialized \ac{VNN} with inputs represented as the discrete spatial distribution $f[\bm{c}]$ within $\bm{C}$, we convolve the field $\bm{C}^{(n)}$ at timestep $t_n$ with $\bm{G}_j$ as $\bm{C}^{(n)}_j=\bm{G}_j\ast \bm{C}^{(n)}$. After which there are two possibilities depending on the Codd state field $\bm{S}$ assuming it is gated to produce $s \in [0, 1]$ per cell. Firstly, there are cells which choose to `pass-through' the signal having been weighted by $\bm{W}$ and will not activate. Secondly, there are cells that decide to activate by $\sigma$ given its weighted neighbours. In the latter, the signals passed through by neighbouring cells already have the weights applied, so that the next timestep $t^{(n+1)}$ can be given by
\begin{equation}\label{eqn::step}
	\bm{C}^{(n+1)} = \frac{1}{M}\sum_{j}^{M-1} \left[ \left(\bm{I}-\bm{S}\right)\cdot\left(\bm{C}^{(n)}_j\cdot \bm{W}\right) + \bm{S}\cdot\sigma\left(\bm{C}^{(n)}_j + \bm{B} \right) \right],
\end{equation}
where $M$ is the total number of Green's functions with $0 \leqslant j < M$, $\cdot$ and $\ast$ represent the element-wise (Hadamard) product and convolution operations respectively, and $\bm{I}$ is the identity field (i.e. all ones). The total number of steps depends on the desired output locations and size of the Green's kernel with the maximal distance that such a kernel can cover is $\nicefrac{(k-1)}{2}$, where $k$ is the kernel size (assuming $k$ is odd and $\bm{C}$ size to be even) in the direction of propagation to ensure information is interconnected between each step. Note that the averaging across $j$ does not need to be done per step and can be done at the final step for maximal feature representation at the cost of computational complexity and memory. At first it may seem that need for these steps seem unusual, but they effectively represent the computation of each ``layer'' of the network analogous to the matrix-vector multiplication in a deep learning \acp{MLP} (see Appendix~\ref{sec::correspondence}), the structure of which can be controlled depending on how the Green's functions are applied to form different \ac{VNN} architectures.

\textbf{Architectural Forms.} One of the chief advantages of \acp{VNN} is the simplicity of representing network architectures, especially for \acp{MLP} in arbitrary dimensions as their connectivity in this space have the same forward (and thus backward) operators regardless of its dimension. In this work, we present two forms of \acp{VNN} based on either assuming that the inputs and outputs have an \ac{MLP}-like structure as hyperplanes (i.e. for a 2D field $C$, these are row or column vectors) or not making any assumptions about these structures at all. In each case, to form the actual network architectures themselves, we need to either hand-craft the Codd states $s$ into the desired form or learn the architecture itself given a particular kernel shape and application of the Green's functions that dictate the propagation of the neural signals and therefore define the local connectivity of the Von Neumann neurons. A discussion of the former is made in Appendix~\ref{sec::correspondence} to show the correspondence of the proposed theory to the well established deep learning theory. The latter is discussed in the remainder of this section.

Let the $h^{(n)}(c)$ be a function of cells $c$ at layer $n$ in a hyperplane structure of dimension $d-1$ for a Chua field $\bm{C}$ of dimension $d$ or simply $h^{(n)}$ for brevity. We assume that the forward pass is towards the ascending indices in the channel dimension and adopt the channel last convention. We also limit the Green's function $\bm{G}$ to this structure, but with a width $\omega$ in the channel dimension to obtain a hyperplane layer specific Green's functions $\widebar{\bm{G}}_j$ with a step distance of $\omega$ for each layer. The same structure can be applied to the weights $\bm{W}$, biases $\bm{B}$ and states $\bm{S}$ to give $\widebar{\bm{W}}$, $\widebar{\bm{B}}$ and $\widebar{\bm{S}}$. Then the forward pass step is also limited to $\omega$ convolutions of hyperplanes of dimension $d-1$, so that the next layer $h^{(n+1)}$ is given by
\begin{eqnarray}\label{eqn::step_hplane}
	h^{(n+1)} = \frac{1}{M}\sum_{j}^{M-1} \left[ \left(\widebar{\bm{I}}-\widebar{\bm{S}}\right)\cdot\left(\widebar{\bm{G}}_j\ast h^{(n)}\cdot \widebar{\bm{W}}\right) + \widebar{\bm{S}}\cdot\sigma\left(\widebar{\bm{G}}_j\ast h^{(n)} + \widebar{\bm{B}} \right) \right],
\end{eqnarray}
with convolution having the highest operator precedence. For the case that no structure can be or is assumed for the network and its convolution, the full Green's function is learnt and the convolution is always full $d$ dimensional, therefore the layers of the network are also $d$ dimensional in shape and the forward step is given by~\eqref{eqn::step}. Additional details for the forward step and model training procedure can be found in Appendix~\ref{sec::Algorithms}. We may also incorporate information bottlenecks and other constraints on cells and the architecture through setting those states to $-\infty$ similar to its use in triangular matrices for computing attention in transformers. In the next section, we present the experiments conducted and theorem proved in this work to demonstrate the \ac{VNN} in its various forms.

\textbf{Experiments.} To show that \acp{VNN} can support not only regular \ac{MLP} vector like data, we construct \acp{CLP} and tape-like \acp{MLP} models show the various different configurations of the inputs and outputs, and thus the networks, are possible in solving binary arithmetic including exclusive OR (XOR) operations. Models were trained to do 8-bit arithmetic between two inputs that result in an output of the same bit depth. To show that multi-dimensional \ac{VNN} \acp{MLP} are simple and scalable, we present image classification results for the simple MNIST and CIFAR10 image datasets compared to regular deep learning \acp{MLP}. A simple three layer (256, 128, 32) and five layer (512, 256, 128, 32, 16) traditional deep learning \acp{MLP} models were compared with Hyperplane ($d-1$ dimensional) and full ($d$ dimensional) \ac{VNN} models for MNIST and CIFAR10 respectively. 

To show that the \ac{VNN} is capable of learning a unified model for all arithmetic and logic, we also create an \acf{ALU} model to demonstrate its ability to use the full $d$ dimensional space of the Chua field $C$ at once rather than the hyperplane structure. We compare a traditional \ac{MLP} with both hyperplane and full \acp{VNN} for 8-bit arithmetic and logic given by a 4-bit control code to produce a single 8-bit output with carry bits ignored in this study. For the full model, we compare two configurations: one with a box like shape with the control code orthogonal to the inputs/outputs, so that most of the model is contained compactly within the space between them; and a tape-like model where all controls, inputs and outputs are configured in a line together. In the last experiment, we push the advantage \acp{VNN} of learning architectures into a network that has not discernible layers because the entire space is a 1D model on a 1D Chua field $C$ with 1D Green's functions to emulate a \ac{TM} like architecture with a tape like model, but with the Green's function acting like a learnable transition table or program. We use this model as the building block for the theorems of computational universality of \acp{VNN}.

\textbf{Setup.} Implementation of the \Neurocell~framework and \acp{VNN} were performed on the JAX numerical and auto-differentiation framework~\citep{jax2018github} (the code is provided as supplementary material). The experiments were carried out on NVIDIA RTX 3090 \& 2080 Ti GeForce graphics cards (CUDA 12.8) and an AMD Radeon 7900 XTX (ROCm 6.1) graphics card. All results were verified across these cards, and CIFAR experiments were also run on an NVIDIA A100. The JAX \textls{grad} transform was used to perform the backpropagation and the AdamW optimizer~\citep{loshchilov_decoupled_2018} with a cosine decay learning rate scheduler starting from $1e^{-3}$ wherever possible.

\section{Results} 
\begin{figure}[ht!]
	\centering
	\includegraphics[width=\textwidth]{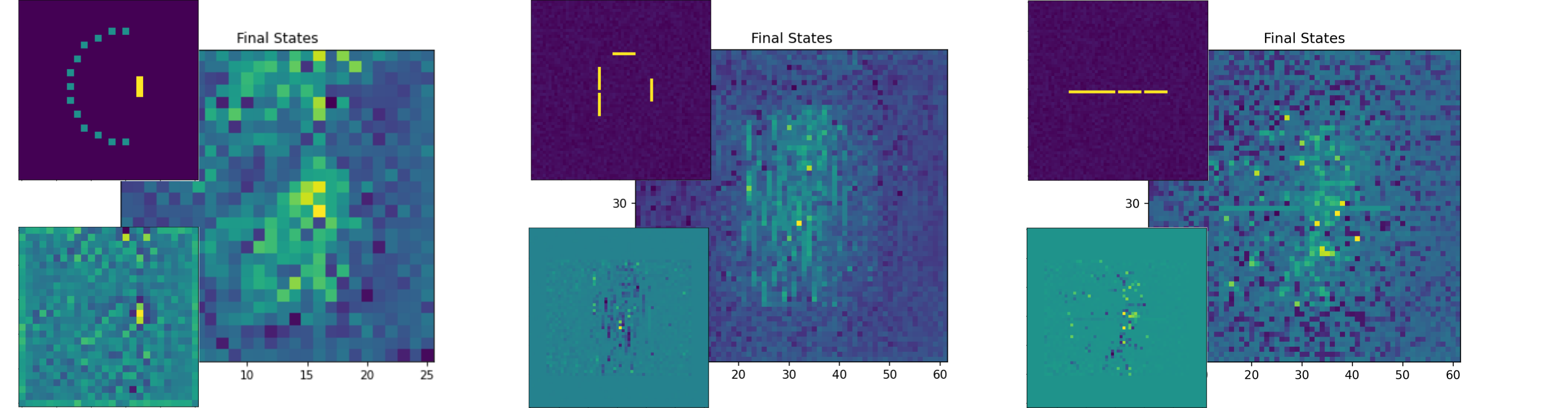}
	\caption{The resulting \acp{VNN} final states (thus the architecture) for \ac{CLP} (left) and two different \ac{ALU} models. The inputs/outputs and controls in the case of the \ac{ALU} models are shown in the top left insets with the bottom right showing the final weights of the models.}
	\label{fig::Results_ALU}
\end{figure}
\textbf{Experiments.} The various types of initial network shapes explored in this work and the resulting \ac{VNN} final model states for the \ac{CLP} and \ac{ALU} models using the full Green's functions approach are shown in Figure~\ref{fig::Results_ALU} (see Appendix~\ref{sec::Greens} for the resulting Green's functions and ablations on learning architecture in Appendix~\ref{sec::Ablations}). The \ac{CLP} model shows \acp{VNN} using a concave shape for the input in a simple classification task~\citep{aeberhard_comparative_1994}. Figure~\ref{fig::Overview} (right) shows the resulting MNIST model and its data-driven encoder-like structure determined automatically during training. Indeed, the \ac{VNN} supports arbitrary shapes for its inputs/outputs as long as there are sufficient steps taken between them to ensure that signals are fully connected. Table~\ref{tab::Results_Table} shows how the \ac{VNN} \ac{MLP} compares against the deep learning \ac{MLP} on various datasets and tasks. Applications of \acp{VNN} to well suited problems is presented in Appendix~\ref{sec::Applications}, including hardware-like repairable networks and simulating the \ac{GoL} within \acp{VNN}.

	%
	%
	%
	%

\begin{table}[htbp]
	\centering
	\small
	\setlength{\tabcolsep}{4pt} 
	\caption{Comparison of traditional MLPs and VNNs performance across various datasets and tasks}
	\label{tab::Results_Table}
	\begin{tabularx}{\linewidth}{@{} l l >{\raggedright\arraybackslash}X r r r r r @{}}
		\toprule
		\textbf{Dataset} & \textbf{Task} & \textbf{Model} & \textbf{Acc. (\%)} & \textbf{\# Params} & \textbf{\shortstack[r]{Inf. Time\\(s)}} & \textbf{\shortstack[r]{Train Time\\(s)}} & \textbf{\shortstack[r]{Throughput\\(samples/s)}} \\
		\midrule
		\rowcolor{gray!20} \textbf{MNIST} & Classify & MLP & 96.5 & 244.5K & 2.8921 & 151.70 & 3,457.66 \\
		& & Logistic Regression & 92.6 & 7.9K & 0.0469 & 156.68 & 213K \\
		\rowcolor{gray!20} & & Random Forest & 96.8 & 707.6K & 0.1251 & 3.91 & 80K \\
		& & VNN Hyperplanes & 96.7 & 41.2K & 4.4076 & 562.57 & 2,268.79 \\
		\midrule
		\rowcolor{gray!20} \textbf{CIFAR10} & Classify & MLP & 62.4 & 11.1M & 11.2762 & 1614.03 & 886.82 \\
		& & VNN 3D Hyperplanes & 71.9 & 841.7K & 15.4264 & 8789.62 & 648.24 \\
		\rowcolor{gray!20} & & VNN 3D Full & 77.7 & 1.2M & 25.7458 & 18750.47 & 388.41 \\
		\midrule
		\rowcolor{gray!20} \textbf{ALU} & Arithmetic & MLP & 66.8 & 2.0K & 2.9845 & 295.85 & 67.01 \\
		& & VNN 1D Hyperplanes & 69.9 & 7.9K & 2.3770 & 661.08 & 84.14 \\
		\rowcolor{gray!20} & & VNN 2D Full & 99.8 & 16.3K & 6.1487 & 13154.22 & 32.53 \\
		\midrule
		\rowcolor{gray!20} \textbf{XOR} & Arithmetic & MLP & 100.0 & 36.2K & 0.7814 & 21.19 & 8,190.13 \\
		& & Linear Regression & 1.1 & 136 & 0.0003 & 0.01 & 20.2M \\
		\rowcolor{gray!20} & & Polynomial LR & 66.3 & 7.8K & 0.0136 & 0.80 & 471K \\
		& & Random Forest & 36.3 & 6.1M & 0.2179 & 12.44 & 30K \\
		\rowcolor{gray!20} & & VNN Hyperplanes & 100.0 & 3.3K & 0.5438 & 113.28 & 91.94 \\
		\bottomrule
	\end{tabularx}
\end{table}

The \ac{ALU} model shown in Figure~\ref{fig::Results_ALU} (middle) is shown in Table~\ref{tab::Results_Table} as \ac{VNN} Full and is able to learn the task that neither the deep learning \ac{MLP} nor the \ac{VNN} \ac{MLP} could. This is likely due to the extra 2D connectivity that the Full \ac{VNN} supports for a 2D Green's function rather than a hyperplane on the Chua field $C$, which was actually predicted by \Johnie~\citep{neumann_computer_1958}. The \ac{ALU} model in Figure~\ref{fig::Results_ALU} (right) shows that this cannot be caused by the 2D layout of the input/outputs and control alone, since all these elements are entirely in 1D in this model, but rather the 2D instructions that appear to be necessary. More importantly, this \ac{ALU} model shows that together with the input/outputs and memory, the Full \ac{VNN} is capable of simulating the entire Von Neumann (hardware) architecture~\citep[first published in 1945]{von_neumann_first_1993} as a \ac{EoC} design as opposed to a traditional \ac{SoC} system. The primary advantage of such a \ac{EoC} design would be the ability to `repair' the chip by fine-tuning the model after designating the damaged cells as $-\infty$ similar to its use in computing attention. Finally, the vertical and horizontal structures imprinted on the final states of both \ac{ALU} models are purely data-driven and requires further exploration.
\begin{figure}[ht!]
	\centering
	\includegraphics[width=\textwidth]{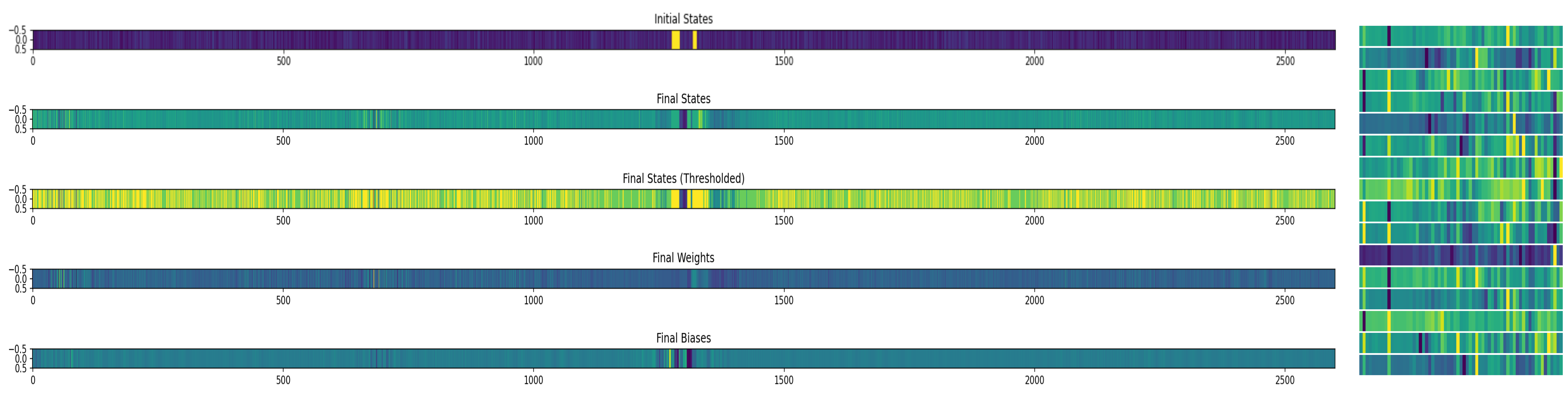}
	\caption{The resulting \ac{VNN} tape based model (purely 1D Chua field $C$ with 1D Green's functions and 1D convolutions) for XOR arithmetic of two 8-bit numbers. The 16 ``instructions'' learnt as 1D Green's functions for the model are shown on the right.}
	\label{fig::Results_Tape}
\end{figure}
Figure~\ref{fig::Results_Tape} shows that fact all one needs is a tape (1D) structure, where the resulting tape based model is from only 1D operations on a 1D field $C$ for 8-bit XOR arithmetic. The resulting model and setup is very reminiscent of a \acf{TM} and the Green's function represent the instructions learnt for the task (see Figure~\ref{fig::Results_Tape} (right)). Additional work is required in studying these instructions and applying them to other unsolved problems in \acp{TM} such as the Busy Beaver~\citep{rado_non-computable_1962}. Lastly, although the CIFAR10 \ac{VNN} results outperform the traditional \ac{MLP}, they are still far from state-of-the-art to \acp{CNN} such as the ResNet that can achieve 94\% accuracy~\citep[see DAWNBench Challenge]{he_deep_2016}. Further work is required in constructing \ac{CNN}-like \acp{VNN} that are translational/rotational equivariant to reach that level of performance.

\textbf{Theorems.} We construct three proofs that show \acp{VNN} are Turing complete. The first of this type can be found in Appendix~\ref{sec::correspondence} that shows \acp{VNN} have a direct correspondence to deep learning and hence are also universal. In general, there are three ways one can show completeness: construct a universal program/machine that can simulate any program; show that it can perform all basic operations that can lead to such a program; construct an actual Turing machine (or an equivalent Turing complete system) within the system itself. Alan Turing constructed the original \ac{UTM} as the first type~\citep{turing_computable_1937} (see Appendix~\ref{sec::TM} for a basic introduction), \Johnie~introduced the Von Neumann (hardware) architecture as second type~\citep{von_neumann_first_1993} and the \ac{GoL} proof~\citep{rendell_fully_2013} is example of the third type. The concepts utilized in this section rely on the theory of computation and a brief basic introduction to \acp{TM} is provided in Appendix~\ref{sec::TM}. In our \ac{VNN}, the \ac{CM} features stochastic gradient descent or manually (see Appendix~\ref{sec::correspondence}), so that it is capable of learning its own architecture.
\begin{definition}[Cellular Machine]\label{def::DFT}
	A \ac{CM} is a system on an infinite array of cells in which each cell maintains its own state and that their connectivity is defined by a neighborhood rules. This function is in general a set of Green's functions that characterizes the machine, and in conjunction with the individual cell states, create the program/transition rules. The inverse of the Green's functions can be used to backpropagate signals and learn its own rules and therefore self-engineer its program.
\end{definition}
\begin{theorem}[Machine-based Universality] \label{thm::Machine}
	The \acf{CM} is equivalent to a \ac{TM} that is Turing complete by being able to simulate the Von Neumann (hardware) architecture.
\end{theorem}
\begin{proof}
	Consider that each of the Von Neumann neurons are simple finite state automata because of its two-state Codd state $s$, i.e.~the neighbourhood is its input and the output to change in its current state. The neuron can then change its internal state $s$ if the neurons are linked together, i.e.~if there is some way to link the inputs and changes in states of each neuron. In the \ac{GoL}, this is achieved via the glider construct that is able to propagate throughout the space and the Gosper glider gun is able to produce an arbitrary stream of such gliders acting as a signal generator. \Johnie~is able to achieve this through explicitly constructing a signal generator construct and Wolfram is able to do this in Rule 110 because of propoagating signals with his 1D CA. The VNN has propagation built into it by design through the use of Green's functions, which also known as propagators in physics. The Green's functions provides the link between neurons, causing them to change their states in accordance to the weights/biases of the network that is optimized through stochastic gradient descent.
	
	We can then observe that a \ac{CM} is equivalent to a \ac{TM} in the following way. The cells themselves are the (multi-dimensional) tape of the \ac{TM}, the learnt states are the states of the \ac{TM}, the Green's functions are the program and the convolutions are analogous to operations of the read/write head moving across the tape. To show that a \ac{CM} is universal and capable of being a \ac{UTM}, consider the \ac{ALU} model presented in Figure~\ref{fig::Results_ALU} (middle) that shows a \ac{CM} can possess input/outputs, the \ac{ALU} and hold memory within its cells, making it equivalent to a Von Neumann (hardware) architecture, which is Turing complete. We provide an AI-assisted LEAN verifed proof to confirm this result in the supplementary material.
\end{proof}

\begin{theorem}[Rules-based Universality] \label{thm::Rules}
	The \acf{CM} is equivalent to a \ac{GoL} system, able to construct a \ac{TM} and is therefore capable of universal computation.
\end{theorem}
\begin{proof}
	We begin by noting that \Johnie's original formulation of a cellular system was for self-replication that was implemented as a \ac{TM} with a program for replication and is therefore Turing complete~\citep{von_neumann_theory_1966}. We showed that (1) a \ac{VNN} is possible that learns a (continuous) two-state system as a contrastive learning process between the state 0 and 1 (see Figure~\ref{fig::vnn}) similar to the \ac{GoL}, (2) it is possible to learn all of the rules of \ac{GoL} as a neighborhood function, since it only involves the 8-connected neighborhood and (3) it is then possible to build the \ac{UTM} directly on the cells using these two-state rules as shown similarly for the \ac{GoL}~\citep{rendell_fully_2013} (see Appendix~\ref{sec::GoL}). Evidence for learning the \ac{GoL} rules are provided in Appendix~\ref{sec::gol-results} showing that the Green's functions can approximate them.
\end{proof}

\textbf{Discussion.}
The unique properties of \acp{VNN} with respect to related works are that they are the only fully differentiable, continuous learning capabilities of Neural Operators, while reducing the Neural \ac{CA} requirements of an \ac{MLP} per cell to just a single state parameter and use convolutions rather than more challenging logic gates to grow data-driven network architectures. The architecture ablations in Appendix~\ref{sec::Ablations} show that \acp{VNN} only need a minimum grid and kernel size for the problem at hand, and the convergence rate is dictated by the number of Green's functions. A minimum grid and kernel sizes are required to hold the inputs and connect the neurons, upon which any size greater can be used at the cost of efficiency and computation time. In fact, Figure~\ref{fig::Results_Tape} that shows a \ac{VNN} tape based model, where the dimensionality of the cellular array need not be a hyper-parameter either, which fits in accordance and corresponds to computational theory that only needs 1D tape to accomplish all computation.

\textbf{Limitations and Future Work.}
The current \Neurocell~framework relies on learning Green's functions using convolutions, but the convolution kernel sizes tend to be large ($> 13$). Better support for large kernel sizes need to be explored, such as scaling methods~\citep{ding_scaling_2022} or dilated approximations~\citep{guo_visual_2023}. The states being optimized are not strictly discrete like \acp{CA}, perhaps quantized states could improve efficiency~\citep{hwang_fixed-point_2014} and a study similar to~\citet{li_training_2017} would help. Further suggested studies include exploring/adapting to other deep learning models such as transformers and \acp{CNN}, extension of the Von Neumann neuron to other topologies, where convolution operators exist such as graphs~\citep{kipf_semi-supervised_2017}, as well as learning to automatically construct self-replicating systems and artificial life.

\section*{Conclusion}
This work presented \acp{VNN} that uses state-based neurons as cells on an infinite array of cells in any dimension that learn their role in the network allowing for them to self-engineer their own architecture. The result is a computationally universal system capable of networks with unstructured layers that can embed a \ac{TM} with a tape structure and a `healable' Von Neumann (hardware) architecture through model training. Future work will involve creating intricate multi-agent systems, automating \acp{CA} and eventually constructing models that can create self-replicating machines on demand.

\small
\bibliographystyle{elsarticle-harv} 
\bibliography{../neurips_2026/references}


\appendix
\normalsize
\section{Brief Review of Self-replication}\label{sec::review-replication}
John von Neumann used 29 different states in his original formulation to imbue his cellular neurons with special functions during the self-replication process of his machine, where the machine encodes instructions of a universal Turing machine that eventually computes the replication~\citep{von_neumann_theory_1966}. Significant progress has been made in implementing~\citep{mange_macroscopic_2004,buckley_signal_2008} and extending self-replicating machines of von Neumann in the decades following his work. \citet{codd_cellular_1968}, with subsequent recent amendments from~\citet{hutton_codds_2010}, reduced the requirement of self-replication requiring 29 states to a total of only 8 states. Others have further simplified self-replication to fewer states, reduced the dimensionality, and removed universal requirements~\citep{banks_information_1971,wolfram_statistical_1983,langton_self-reproduction_1984}. A more detailed review of this area can be found in~\citet{sipper_fifty_1998}.

More recent work has focused on simpler \ac{CA} with as few rules as possible, while maintaining complex behaviour. The most famous of these is John Conway's Game of Life~\citep{gardner_mathematical_1970}, where the cells have only 2 states and rules that change the state of cells based on its neighbours. Surprisingly, even this simple 2-state cellular system is Turing complete~\citep{rendell_fully_2013}. More complex rules have been developed to simulate artificial life~\citep{hong_kong_lenia_2019} and are the subject of ongoing research~\citep{chan_lenia_2020}.

\section{Related Works Comparison}\label{sec::Comparison}
The proposed VNNs are unique in that we show that a neural network is possible where a simple mechanism, namely a singular learnt parameter called the Codd state $s$ within a neuron allows that neuron to adapt its role in a network, enabling a data-driven architecture to be trained. In the table below we precisely highlight these differences between VNNs and related works.

{\def\LTcaptype{table} 
	\small
	\begin{longtable}[]{@{}
			>{\raggedright\arraybackslash}p{0.1032\linewidth}
			>{\raggedright\arraybackslash}p{0.1109\linewidth}
			>{\raggedright\arraybackslash}p{0.1195\linewidth}
			>{\raggedright\arraybackslash}p{0.1021\linewidth}
			>{\raggedright\arraybackslash}p{0.1132\linewidth}
			>{\raggedright\arraybackslash}p{0.1032\linewidth}
			>{\raggedright\arraybackslash}p{0.1604\linewidth}@{}}
		\toprule\noalign{}
		\begin{minipage}[b]{\linewidth}\raggedright
			Model / Approach
		\end{minipage} & \begin{minipage}[b]{\linewidth}\raggedright
			Continuous vs Discrete
		\end{minipage} & \begin{minipage}[b]{\linewidth}\raggedright
			Differentiability
		\end{minipage} & \begin{minipage}[b]{\linewidth}\raggedright
			Hardware Efficiency
		\end{minipage} & \begin{minipage}[b]{\linewidth}\raggedright
			Interpretability
		\end{minipage} & \begin{minipage}[b]{\linewidth}\raggedright
			Architecture Learning
		\end{minipage} & \begin{minipage}[b]{\linewidth}\raggedright
			Core Mechanism
		\end{minipage} \\
		\midrule\noalign{}
		\endhead
		\bottomrule\noalign{}
		\endlastfoot
		\textbf{VNNs} & Continuous (can approximate discrete) & Fully
		Differentiable & High (maps well to parallel \& spatial hardware) &
		Moderate (visualizable spatial weights/states) & Yes (cells learn and
		heal) & Learnable Green's functions \& cell Codd states \\
		\textbf{Neural Operators} & Continuous & Fully Differentiable & Moderate
		to High & Low & No (Fixed Topology) & Mapping between
		infinite-dimensional function spaces \\
		\textbf{Cellular Automata (e.g., GoL)} & Discrete & Non-Differentiable &
		High (simple bitwise operations) & High (rules are explicit) & No
		(Emergent patterns, but rules are fixed) & Hand-crafted local transition
		rules \\
		\textbf{CeNNs} & Continuous (Analog) & Differentiable (ODE based) & High
		(Analog VLSI) & Moderate & No (Fixed templates) & Analog ODEs with local
		cloning templates \\
		\textbf{Neural CA} & Continuous & Fully Differentiable & Low (MLP per
		cell is heavy) & Low & Yes (Morphogenesis/Healing) & Per-cell MLPs
		predicting local state changes \\
		\textbf{Diff. Logic Gates} & Continuous (train) / Discrete (infer) &
		Differentiable (via relaxations) & Very High (pure logic
		inference) & High (resolves to logic circuits) & No (Usually searches
		within fixed super-graph) & Continuous relaxation of Boolean
		operations \\
		\textbf{Random Forests} & Discrete & Non-Differentiable & Moderate
		(Branching paths) & Very High & Yes (Tree structure is grown) &
		Information Gain / Gini impurity splits \\
		\textbf{State Space Models} & Continuous & Fully Differentiable & High
		(Linear recurrence) & Low & No (Fixed state sizes and sequence flow) &
		Continuous-time ODEs for temporal sequences \\
	\end{longtable}
}

As far as the author is aware, none of these systems is able to utilize a single 1D model for all layers and weights/biases like the VNNs as shown in 1D tape model in Figure 4. The 1D CA of Wolfram is similar in nature but is non-differentiable. In Figure 4 of the paper, the VNN tape model embeds the entire network topology: inputs, outputs, intermediate representations, weights/biases into a single, unified 1D spatial array (the 1D Chua field). The computation then temporally unfolds through 1D convolutions, where the learnt 1D Green's functions act as the transition rules, blurring the line between spatial layers and time like a Turing Machine tape. In otherwords, a VNN can dynamically learn to route signals back and forth across a single 1D space. All other models require at best multiple distinct layers of additional 1D or 2D blocks stacked together to form their respective models.

\textbf{Significance.} VNNs have the potential to be one of the few ways to achieve neuron-level yet hardware-native form of artificial intelligence capable of self-engineering and healing (see Appendix~\ref{sec::Applications} for evidence of some of these applications).

\section{Algorithms and Training Procedure}\label{sec::Algorithms}
We provide the training procedure with algorithm tables for the forward pass operator in Algorithm~\ref{alg::forward} and the training procedure in Algorithm~\ref{alg::training} as shown below, as well as tables for the architectural choices.

\begin{algorithm}[H]
	\small
	\caption{VNN Forward Pass}
	\begin{algorithmic}[1]\label{alg::forward}
		\REQUIRE Cellular field $C^{(0)}$, Weights $W$, Biases $B$, States $S$, Green's functions $\{G_j\}_{j=1}^M$, Depth $d$
		\REQUIRE Activation function $\sigma$
		\FOR{$n = 0$ \TO $d-1$}
		\STATE Initialize $C^{(n+1)} \leftarrow 0$
		\FOR{$j = 1$ \TO $M$}
		\STATE Compute propagated field: $C_j^{(n)} = G_j * C^{(n)}$
		\STATE Apply standardization/layer normalization (optional) on $C_j^{(n)}$
		\STATE Compute passive signal: $P_j = (I - S) \cdot (C_j^{(n)} \cdot W)$
		\STATE Compute active decision: $A_j = S \cdot \sigma(C_j^{(n)} + B)$
		\STATE Blend states: $F_j = P_j + A_j$
		\STATE Accumulate: $C^{(n+1)} \leftarrow C^{(n+1)} + F_j$
		\ENDFOR
		\STATE Average across filters: $C^{(n+1)} \leftarrow \frac{1}{M} C^{(n+1)}$
		\ENDFOR
		\RETURN $C^{(d)}$
	\end{algorithmic}
\end{algorithm}

Algorithm~\ref{alg::forward} shows that VNNs are trained using standard Stochastic Gradient Descent (SGD) and backpropagation, treating the cellular machine as a differentiable computation graph. During training, the backward pass simulates the inverse linear differential operator $\mathcal{D}^{-1}$, distributing gradients back through the continuous cellular field.

\begin{algorithm}[H]
	\small
	\caption{VNN Training Procedure}
	\begin{algorithmic}[1]\label{alg::training}
		\REQUIRE Training dataset $\mathfrak{D}$, Learning rate schedule, Optimizer (e.g., AdamW)
		\STATE \textbf{Initialization:}
		\STATE Initialize cellular universe field $C \leftarrow 0$
		\STATE Initialize Weights $W \sim \mathcal{N}(\mu, \sigma^2)$ and Biases $B \leftarrow 0$
		\STATE Initialize States $S$ randomly, then bound via sigmoid $S \leftarrow \sigma_{sig}(S_{raw})$
		\STATE Fix $S \leftarrow 1.0$ at input and output coordinates
		\STATE Initialize Green's functions $G_j$
		\WHILE{not converged}
		\STATE Sample a minibatch of inputs $X$ and targets $Y$ from $\mathfrak{D}$
		\STATE Place $X$ at designated input coordinates in $C^{(0)}$
		\STATE Compute forward pass $\hat{Y} = \text{VNN\_Forward}(C^{(0)}, W, B, S, \{G_j\})$ sampled at output coordinates
		\STATE Compute loss $\mathcal{L}(\hat{Y}, Y)$
		\STATE Compute gradients $\nabla_{W, B, S_{raw}, G_j} \mathcal{L}$ via reverse-mode auto-differentiation
		\STATE Update parameters $W, B, S_{raw}, G_j$ using Optimizer
		\ENDWHILE
		\RETURN $W, B, S, \{G_j\}$
	\end{algorithmic}
\end{algorithm}

\paragraph{Architectural Design}\label{architectural-design}

{\def\LTcaptype{table} 
	\small
	\begin{longtable}[]{@{}
			>{\raggedright\arraybackslash}p{0.2200\linewidth}
			>{\raggedright\arraybackslash}p{0.2200\linewidth}
			>{\raggedright\arraybackslash}p{0.2200\linewidth}
			>{\raggedright\arraybackslash}p{0.2200\linewidth}@{}}
		\toprule\noalign{}
		\begin{minipage}[b]{\linewidth}\raggedright
			Parameter / Decision
		\end{minipage} & \begin{minipage}[b]{\linewidth}\raggedright
			MNIST VNN MLP
		\end{minipage} & \begin{minipage}[b]{\linewidth}\raggedright
			CIFAR10 VNN 3D Hyperplane
		\end{minipage} & \begin{minipage}[b]{\linewidth}\raggedright
			VNN ALU Chip
		\end{minipage} \\
		\midrule\noalign{}
		\endhead
		\bottomrule\noalign{}
		\endlastfoot
		\textbf{Dataset} & MNIST & CIFAR-10 (via Grain Dataloader) & Synthetic
		ALU Arithmetic (128k samples) \\
		\textbf{Epochs} & 8 & 35 & 200 \\
		\textbf{Batch Size} & 64 & 100 & 128 \\
		\textbf{Layers} & 4 & 20 & N/A \\
		\textbf{Network Depth} & 9 (\texttt{layers\ *\ 2\ +\ 1}) & 40
		(\texttt{layers\ *\ 2}) & 4 \\
		\textbf{Kernel Size} & \texttt{(71,\ 3)} & \texttt{(9,\ 9,\ 41)}
		(channel first, depth last) & \texttt{(23,\ 13)} \\
		\textbf{Directions (Green's function)} & 2 & 32 & 16 \\
		\textbf{Universe Grid Size ($N$)} & $1024 \times 17$
		($N_x \times N_y$) & $44 \times 44 \times 89$
		($N_x \times N_y \times N_z$) & $62 \times 62$ (Square Chip
		Layout) \\
		\textbf{Optimizer} & Adam & AdamW & AdamW \\
		\textbf{Learning Rate} & \texttt{0.01} (Constant) & Warmup Cosine Decay
		(\texttt{peak=0.03}, \texttt{init=0.001}, \texttt{end=0.0005}) & Warmup
		Cosine Decay (\texttt{peak=0.03}, \texttt{init=0.001},
		\texttt{end=0.0005}) \\
		\textbf{Activation (Hidden)} & \texttt{jax.nn.tanh} &
		\texttt{jax.nn.leaky\_relu} & \texttt{jax.nn.tanh} \\
		\textbf{Activation (Final)} & \texttt{jax.nn.sigmoid} &
		\texttt{jax.nn.leaky\_relu} & None \\
		\textbf{Loss Function} & Softmax Cross Entropy
		(\texttt{optax}) & Softmax Cross Entropy
		(\texttt{optax}) & Mean Squared Error
		(\texttt{optax}) \\
		\textbf{Input Shape Mapping} & 1D Vector (\texttt{28*28,\ 1}) & 2D/3D
		Tensor mapped into field & Sequential 1D Inputs ($X_1, X_2, C$) mapped
		via offsets \\
	\end{longtable}
}

Architectural design choices for the XOR model is provided in the
ablation appendix.

\section{Ablation Studies}\label{sec::Ablations}
From the feature comparison table, one can observe that VNNs and Neural
CAs are the only fully differentiable systems capable of learning their
own architecture. To show clarify what parts of the architecture are
free variables and which are hyper-parameters, we present the ablations
below on a XOR 2D VNN MLP model with the details below.

\paragraph{XOR VNN Ablation Model
	Details}\label{xor-vnn-ablation-model-details}

\begin{itemize}
	\small
	\item
	\textbf{Base Grid (Universe) Size}: $N \times N$, where baseline
	$N = 50$.
	\item
	\textbf{Depth}: $layers = 5$, resulting in an internal logical depth
	of $11$ ($2 \times layers + 1$) representing alternating decision
	and signal layers, plus an output layer.
	\item
	\textbf{Activations}: Internal states utilize \texttt{jax.nn.tanh} for
	non-linear activations.
	\item
	\textbf{Layout}: Inputs and outputs are spatially separated and offset
	from the center of the grid.
	\item
	\textbf{Samples}: $32,000$ total generated samples.
	\item
	\textbf{Inputs}: The inputs are unpacked into 8-bit arrays and
	continuous noise ($\delta = 0.2$).
	\item
	\textbf{Outputs}: Raw binary XOR results.
	\item
	\textbf{Data Split}: Standard 80/20 train-test split.
	\item
	\textbf{Epochs}: $100$
	\item
	\textbf{Batch Size}: $128$
	\item
	\textbf{Optimizer}: \texttt{optax.adamw}
	\item
	\textbf{Learning Rate Schedule}: Warmup cosine decay
	
	\begin{itemize}
		\item
		Initial LR: $0.001$
		\item
		Peak LR: $0.03$
		\item
		Warmup Steps: $2.5\%$ of total steps
		\item
		End LR: $0.0005$
	\end{itemize}
	\item
	\textbf{Loss Function}: Mean Squared Error (MSE)
\end{itemize}

\paragraph{Ablation: Directions}\label{ablation-directions}

{\def\LTcaptype{table} 
	\small
	\begin{longtable}[]{@{}
			>{\raggedright\arraybackslash}p{0.12\linewidth}
			>{\raggedright\arraybackslash}p{0.12\linewidth}
			>{\raggedright\arraybackslash}p{0.12\linewidth}
			>{\raggedright\arraybackslash}p{0.12\linewidth}
			>{\raggedright\arraybackslash}p{0.12\linewidth}
			>{\raggedright\arraybackslash}p{0.12\linewidth}@{}}
		\toprule\noalign{}
		\begin{minipage}[b]{\linewidth}\raggedright
			Value
		\end{minipage} & \begin{minipage}[b]{\linewidth}\raggedright
			Accuracy
		\end{minipage} & \begin{minipage}[b]{\linewidth}\raggedright
			Total Time (s)
		\end{minipage} & \begin{minipage}[b]{\linewidth}\raggedright
			Time per Epoch (s)
		\end{minipage} & \begin{minipage}[b]{\linewidth}\raggedright
			Params Count
		\end{minipage} & \begin{minipage}[b]{\linewidth}\raggedright
			Model Size (MB)
		\end{minipage} \\
		\midrule\noalign{}
		\endhead
		\bottomrule\noalign{}
		\endlastfoot
		2 & 0.9995 & 104.99 & 1.0499 & 7602 & 0.0385 \\
		4 & 0.9997 & 100.06 & 1.0006 & 7704 & 0.0389 \\
		Baseline (8) & 1.0000 & 141.61 & 1.4161 & 7908 & 0.0397 \\
		12 & 0.9844 & 140.87 & 1.4087 & 8112 & 0.0405 \\
		16 & 0.9998 & 153.96 & 1.5396 & 8316 & 0.0413 \\
	\end{longtable}
}

\paragraph{Ablation: Grid Size}\label{ablation-grid-size}

{\def\LTcaptype{table} 
	\small
	\begin{longtable}[]{@{}
			>{\raggedright\arraybackslash}p{0.12\linewidth}
			>{\raggedright\arraybackslash}p{0.12\linewidth}
			>{\raggedright\arraybackslash}p{0.12\linewidth}
			>{\raggedright\arraybackslash}p{0.12\linewidth}
			>{\raggedright\arraybackslash}p{0.12\linewidth}
			>{\raggedright\arraybackslash}p{0.12\linewidth}@{}}
		\toprule\noalign{}
		\begin{minipage}[b]{\linewidth}\raggedright
			Value
		\end{minipage} & \begin{minipage}[b]{\linewidth}\raggedright
			Accuracy
		\end{minipage} & \begin{minipage}[b]{\linewidth}\raggedright
			Total Time (s)
		\end{minipage} & \begin{minipage}[b]{\linewidth}\raggedright
			Time per Epoch (s)
		\end{minipage} & \begin{minipage}[b]{\linewidth}\raggedright
			Params Count
		\end{minipage} & \begin{minipage}[b]{\linewidth}\raggedright
			Model Size (MB)
		\end{minipage} \\
		\midrule\noalign{}
		\endhead
		\bottomrule\noalign{}
		\endlastfoot
		40 & 0.9998 & 113.88 & 1.1388 & 5208 & 0.0260 \\
		Baseline (50) & 1.0000 & 141.61 & 1.4161 & 7908 & 0.0397 \\
		60 & 0.9998 & 157.90 & 1.5790 & 11208 & 0.0565 \\
		70 & 0.9998 & 189.32 & 1.8932 & 15108 & 0.0763 \\
		80 & 1.0000 & 205.79 & 2.0579 & 19608 & 0.0992 \\
	\end{longtable}
}

\paragraph{Ablation: Kernel Size}\label{ablation-kernel-size}

{\def\LTcaptype{table} 
	\small
	\begin{longtable}[]{@{}
			>{\raggedright\arraybackslash}p{0.16\linewidth}
			>{\raggedright\arraybackslash}p{0.12\linewidth}
			>{\raggedright\arraybackslash}p{0.12\linewidth}
			>{\raggedright\arraybackslash}p{0.12\linewidth}
			>{\raggedright\arraybackslash}p{0.12\linewidth}
			>{\raggedright\arraybackslash}p{0.12\linewidth}@{}}
		\toprule\noalign{}
		\begin{minipage}[b]{\linewidth}\raggedright
			Value
		\end{minipage} & \begin{minipage}[b]{\linewidth}\raggedright
			Accuracy
		\end{minipage} & \begin{minipage}[b]{\linewidth}\raggedright
			Total Time (s)
		\end{minipage} & \begin{minipage}[b]{\linewidth}\raggedright
			Time per Epoch (s)
		\end{minipage} & \begin{minipage}[b]{\linewidth}\raggedright
			Params Count
		\end{minipage} & \begin{minipage}[b]{\linewidth}\raggedright
			Model Size (MB)
		\end{minipage} \\
		\midrule\noalign{}
		\endhead
		\bottomrule\noalign{}
		\endlastfoot
		(5, 3) & 0.0600 & 93.29 & 0.9329 & 7620 & 0.0386 \\
		(9, 3) & 0.9995 & 113.23 & 1.1323 & 7716 & 0.0390 \\
		Baseline (17,3) & 1.0000 & 141.61 & 1.4161 & 7908 & 0.0397 \\
		(13, 3) & 1.0000 & 131.87 & 1.3187 & 7812 & 0.0393 \\
		(21, 3) & 0.9998 & 185.26 & 1.8526 & 10596 & 0.0533 \\
	\end{longtable}
}

\section{Applications}\label{sec::Applications}
In this work, we show two possible applications of \acp{VNN} to demonstrate the types of problems suitable to solve that include repairable networks and emulating the \ac{GoL}.

\subsection{Repairable Networks}\label{sec::repairable-networks}
We believe the VNN framework is uniquely well-suited for hardware-native implementations of models that can adapt their architectures by embedding themselves directly in the topology itself. Our experiments on repairable networks below demonstrate that a VNN architecture can be repaired when damaged in applications, such as corrupted computational elements of a CPU in damaged spacecraft or an autonomous
vehicle, through simple fine-tuning despite losing a small percentage of its neurons.

\paragraph{Data Pipeline}
The model is trained on a synthesized bitwise arithmetic and logic dataset on the ALU task. Each sample consists of two 8-bit input operands ($X_1, X_2$) and a control sequence ($C$).

\paragraph{Simulating Hardware Damage}
To simulate physical hardware failure or radiation-induced bit-flips, we introduce ablations to a fully pre-trained VNN acting as an ALU. For a specified damage level (e.g., $1\%$ to $5\%$ of the total grid), we randomly select a subset of the VNN's internal cells. The Codd states ($s$) of these cells are forcefully overwritten to $-\infty$ (similarly to masked attention in transformers), while their corresponding weights ($w$) and biases ($b$) are zeroed out. We apply a protective mask over the input, control, and output interface cells to simulate damage to the internal routing and mirror a scenario where hardware is partially corrupted but inputs remain intact. In
practice, inputs and outputs can be moved too, although we did not conduct these experiments due to time constraints of the review process.

\paragraph{Repair and Training Procedure}
Following the ablation, the VNN typically experiences a degradation in arithmetic accuracy. We then initiate a repair phase:
\begin{enumerate}
	\item \textbf{Fine-Tuning:} The damaged VNN is fine-tuned on a new training set for a rapid recovery period of 50 epochs.
	\item \textbf{Evaluation:} The network is evaluated on the test set immediately after damage, and once again after the 50-epoch fine-tuning period to measure the recovery of accuracy.
\end{enumerate}

\paragraph{Results}
\begin{figure}[ht!]
	\centering
	\includegraphics[width=0.8\textwidth]{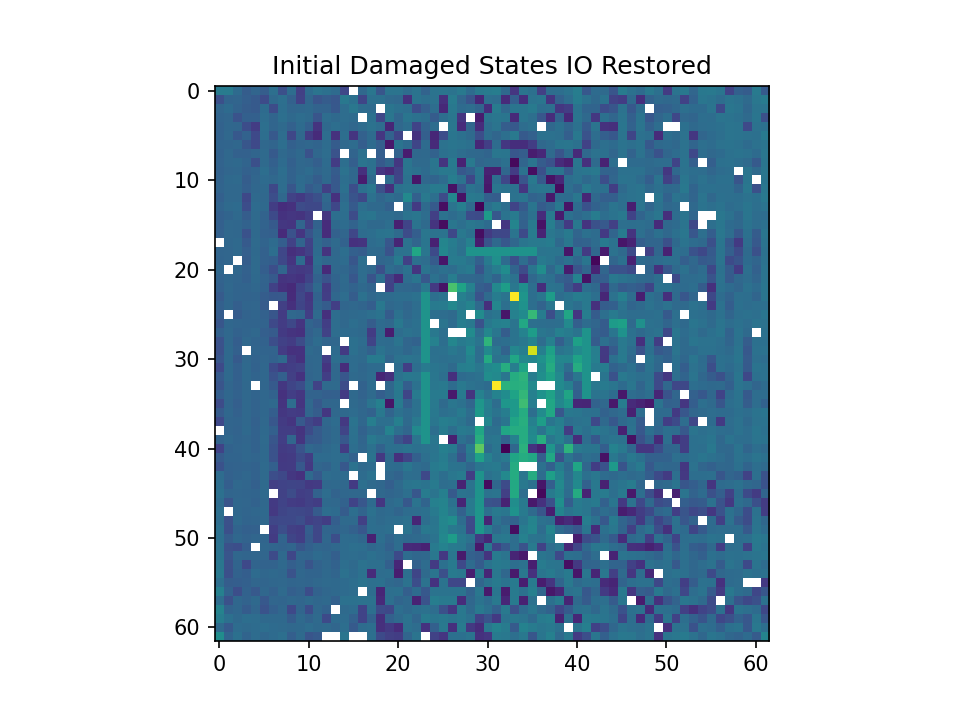}
	\caption{The simulated damage used in experiments showing 3\% of cells whose states designate dead neurons, but with input and output (IO) cells maintained with usual functionality.}
	\label{fig::Damage}
\end{figure}
An example of the type of damage inflicted on the cellular system is shown in Figure~\ref{fig::Damage}. The VNNs were able to be repaired in all instances and the accuracy degradation was proportional to the percentage of damaged cells. The ablations were run across fixed number of epochs during the fine tuning and resulted in the recovery of accuracy, see table below. The repair levels obtained seem to imply that small amounts of damage can be fully restored through fixed fine tuning intervals.

\paragraph{Table: The accuracies of the VNN MLP on the arithmetic task
	before and after 50-epoch repair at varying damage
	levels.}

{\def\LTcaptype{table} 
	\small
	\begin{longtable}[]{@{}
			>{\raggedright\arraybackslash}p{0.1267\linewidth}
			>{\raggedright\arraybackslash}p{0.1267\linewidth}
			>{\raggedright\arraybackslash}p{0.1267\linewidth}
			>{\raggedright\arraybackslash}p{0.1267\linewidth}
			>{\raggedright\arraybackslash}p{0.1267\linewidth}
			>{\raggedright\arraybackslash}p{0.1267\linewidth}@{}}
		\toprule\noalign{}
		\begin{minipage}[b]{\linewidth}\raggedright
			Damage Level
		\end{minipage} & \begin{minipage}[b]{\linewidth}\raggedright
			Accuracy Before Repair
		\end{minipage} & \begin{minipage}[b]{\linewidth}\raggedright
			Accuracy After Repair
		\end{minipage} & \begin{minipage}[b]{\linewidth}\raggedright
			Repair Time (s)
		\end{minipage} & \begin{minipage}[b]{\linewidth}\raggedright
			Params Count
		\end{minipage} & \begin{minipage}[b]{\linewidth}\raggedright
			Model Size (MB)
		\end{minipage} \\
		\midrule\noalign{}
		\endhead
		\bottomrule\noalign{}
		\endlastfoot
		0.00 (Baseline) & 0.9987 & 0.9987 & 0.00 & 16316 & 0.0769 \\
		0.01 & 0.6812 & 0.9984 & 475.48 & 16316 & 0.0769 \\
		0.02 & 0.8619 & 0.9987 & 476.35 & 16316 & 0.0769 \\
		0.03 & 0.3973 & 0.9711 & 472.98 & 16316 & 0.0769 \\
		0.04 & 0.1775 & 0.9981 & 476.69 & 16316 & 0.0769 \\
		0.05 & 0.7674 & 0.9983 & 473.40 & 16316 & 0.0769 \\
	\end{longtable}
}

\subsection{GoL Simulation}\label{sec::gol-results}
To demonstrate the capability of the \acp{VNN} in strictly learning an approximation of the discrete transition rules of \acf{GoL} (see Appendix~\ref{sec::GoL} for a brief introduction), we constructed a dedicated 64x64 grid experiment.

\paragraph{Data Pipeline}
To approximate the 3x3 neighborhood transition rules of \ac{GoL}, a dynamic dataset is generated prior to training following similar findings of~\citet{gilpin_cellular_2019}:
\begin{enumerate}
	\item \textbf{Patterns:} For each sample trajectory, the script randomly selects between five classical \ac{GoL} patterns: \texttt{Blinker} (3x2), \texttt{Glider} (3x3), \texttt{Toad} (4x4), \texttt{Pulsar} (13x13), and \texttt{GliderGun} (10x37).
	\item \textbf{Initial Placement:} The selected pattern is inserted into a blank 64x64 grid with randomized initial coordinates, buffered by their evolution steps (\texttt{l-steps\ =\ 12}) on all boundaries to prevent cropping.
	\item \textbf{Evolution:} The grid is evolved using the ground-truth \ac{GoL} simulator. 
	\item \textbf{Dataset Size:} This process generated 51,200 strictly controlled transition pairs with batch size of 128.
\end{enumerate}

\paragraph{Model Architecture}
A 64x64 grid with kernel size at 3x3 to mimic the \ac{GoL} rules with a depth of 2 having 13440 parameters. A total of 128 independent Green's functions (directions) are learned during training. The \texttt{sigmoid} activation is used for signal propagation, and no activations for the final step. The \ac{VNN} was trained via AdamW with a warmup cosine decay schedule, using an Mean Squared Error (MSE) loss. The training script and the framework will be released open-source with these reviews.

\paragraph{Results}
The model converged by epoch 40 of training, reaching and maintaining 1.0 (100\%) Accuracy on predicting the precise subsequent \ac{GoL}
generation. We show (uneditted) images of the learned 3x3 Green's function routing parameters discovered by the network that seem to encode the \ac{GoL} rules in Figure~\ref{fig::GoL_Filters}. Some of these 3x3 filters show a form of the \ac{GoL} rules by counting neighbours, e.g.~a particular pattern of 3 alive neighbours. Others seem to be combinations of other neighbor operations that require further study. The exact mapping of the 3x3 filters to the \ac{GoL} rules and compression of these rules with respect to the number of Green's functions remains an avenue of future work.

\section{Theory Correspondence}\label{sec::correspondence}
The correspondence principle in physics refers to when a new theory approaches the established theory at an expected limit. In this section, we discuss the correspondence of the \acp{VNN} with respect to deep learning theory. However, before looking at the correspondence in detail, we begin with a simple example of a two neuron \ac{VNN} in a $5\times 5$ universe of cells forward propagated by a $3\times 3$ forward operator or Green's function. Figure~\ref{fig::Example} shows the forward pass of the \ac{VNN} as an algebraic representation to derive the terms that result at each location in the universe during each step of the two convolution and weighting computations.
\begin{figure}[ht!]
	\centering
	\begin{subfigure}[b]{0.95\textwidth}
		\vfill
		\centering
		\includegraphics[width=\textwidth]{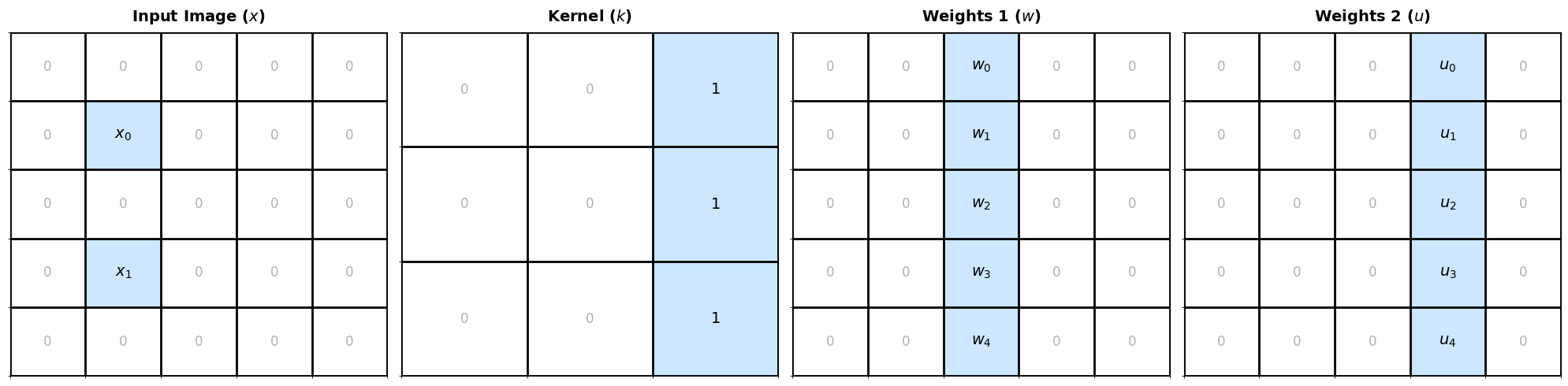}
		\vfill
		\caption{The initial arrays}
		\label{fig::Example_Initial}
	\end{subfigure}
	\begin{subfigure}[b]{0.95\textwidth}
		\vfill
		\centering
		\includegraphics[width=\textwidth]{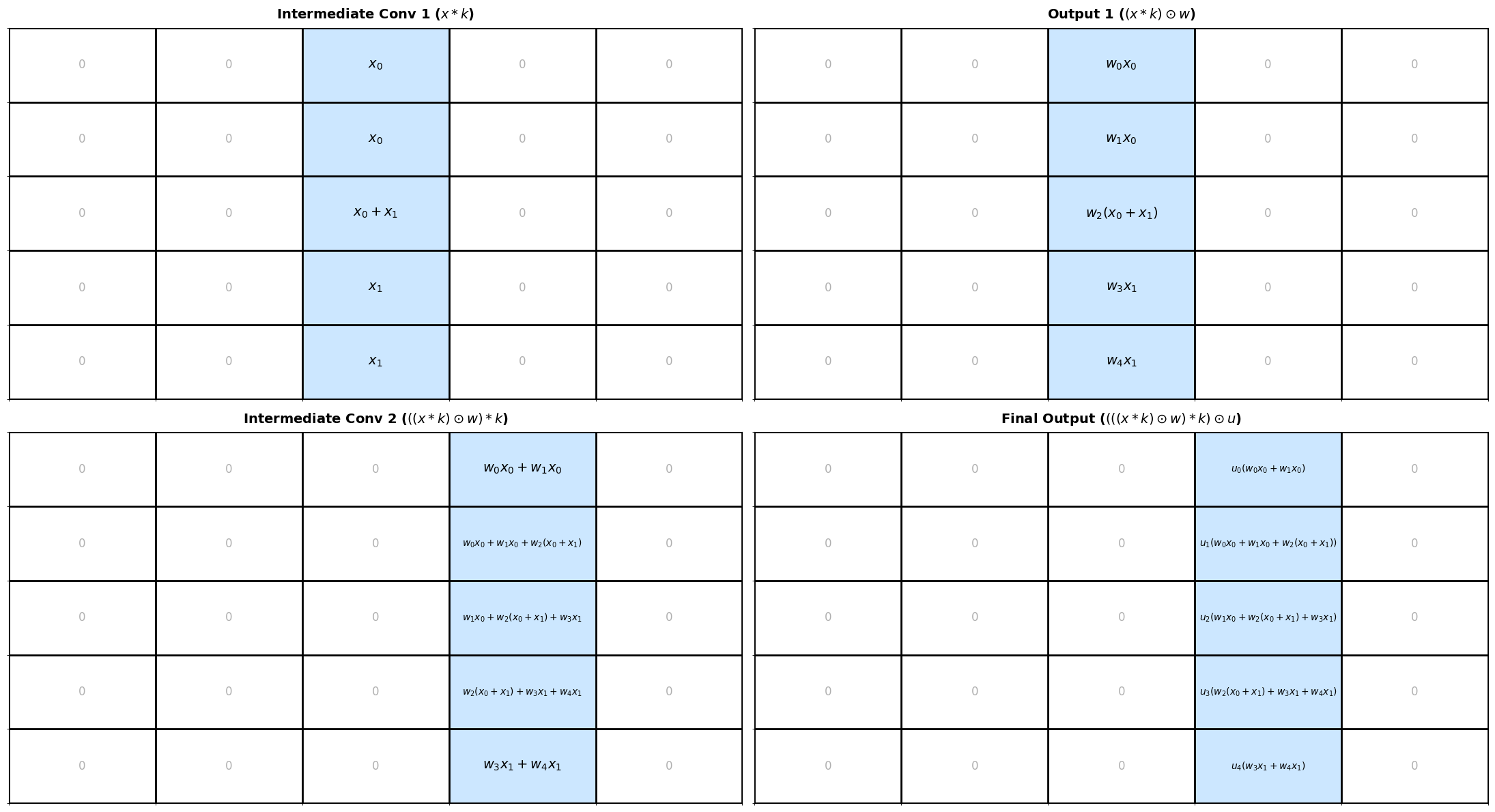}
		\caption{Intermediate and final results}
		\label{fig::Example_Result}
	\end{subfigure}
	\vfill
	\caption{An example of a simple two neuron \ac{VNN} forward pass algebraically. Here we have broken up the weights into separate arrays for illustrative purposes. The biases can be done similarly except through addition rather than multiplication and activations can be applied per layer as needed.}
	\label{fig::Example}
\end{figure}
In this case, the output neuron can be extracted at the coordinate $(2,3)$ in Figure~\ref{fig::Example_Initial} to perform simple tasks such as binary classification etc. It is important to note that not only are there terms from what we would expect from deep learning, but also the existence of second order terms that appear to enhance the \acp{VNN} ability, particularly when using larger multi-dimensional kernels. The interaction of these second, third etc. order terms is determined by the size of the kernel(s) and the arrangement of the inputs within the universe.

To explore this correspondence in detail, let us assume that the forward operator is given by~\eqref{eqn::forward} and the back propagation is conducted according to the backward operator given by~\eqref{eqn::backward}. If we set the Green's function is to be fixed like it is for deep learning because the neurons are given as Figure~\ref{fig::nn} without Codd states $s$, then the linear differential operator $\mathcal{D}$ is also fixed and is given by the Laplacian $\nabla^2$, so that the backward operator becomes the heat equation
\begin{eqnarray}
	\nabla^2 C(c) = f(c).\label{eqn::backward_dl}
\end{eqnarray}
The Green's function to the above equation is a Gaussian function~\citep{arfken_mathematical_2001} that is commonly approximated by a smoothing kernel in practice and its propagation leads to a diffusion process, the most crude of which is the averaging filter consisting of all ones. This is what \Johnie~must have referred to when he predicted that such a system must be of diffusion type in the 1950s. Thus, the standard \ac{MLP} from deep learning literature can be implemented in \acp{VNN} as a system with Green's functions with all ones and hand-crafted Codd states $s$ (that are fixed and represent decision and signal neurons) on the Chua field $C$ as fixed parameters with the appropriate hidden layers etc. to fit the desired architectural approach (see section~\ref{sec::network}).

\begin{figure}[ht!]
	\centering
	\begin{subfigure}[b]{0.4\textwidth}
		\includegraphics[width=\textwidth]{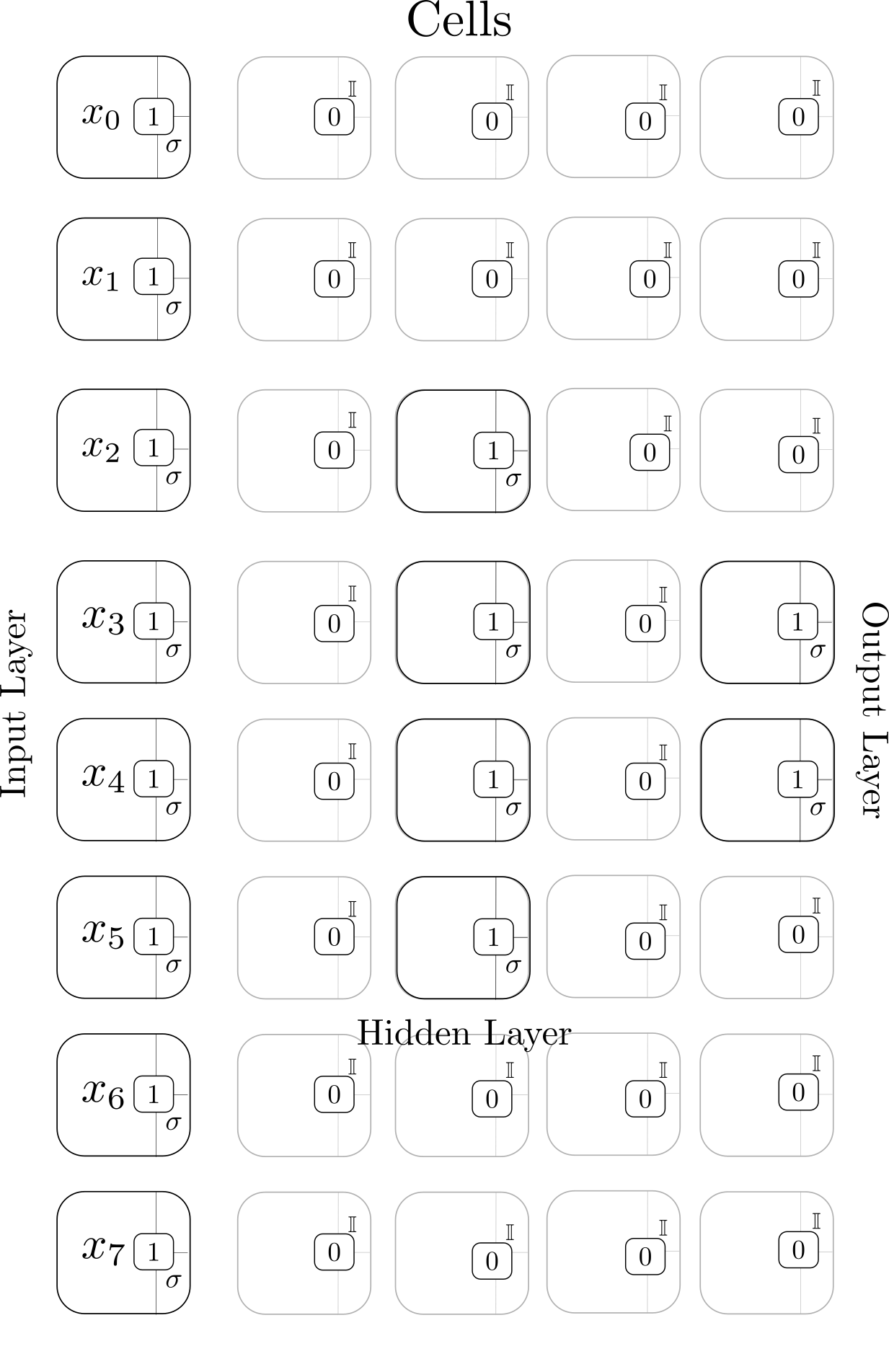}
		\caption{Deep learning like \ac{VNN} \ac{MLP}}
		\label{fig::MLP}
	\end{subfigure}
	\hfill 
	\begin{subfigure}[b]{0.28\textwidth}
		\vfill
		\centering
		\renewcommand{\arraystretch}{1.5} 
		\begin{tabular}{|c|cc|}
			\hline
			1 & \multicolumn{2}{c|}{\multirow{5}{*}{\Huge $\emptyset$}} \\
			1 & \multicolumn{2}{c|}{} \\
			1 & \multicolumn{2}{c|}{} \\
			1 & \multicolumn{2}{c|}{} \\
			1 & \multicolumn{2}{c|}{} \\
			\hline
		\end{tabular}
		\vfill
		\caption{Backward pass operator}
		\label{fig::forward}
	\end{subfigure}
	\hfill 
	\begin{subfigure}[b]{0.28\textwidth}
		\vfill
		\centering
		\renewcommand{\arraystretch}{1.5}
		\begin{tabular}{|cc|c|}
			\hline
			\multicolumn{2}{|c|}{\multirow{5}{*}{\Huge $\emptyset$}} & 1 \\
			\multicolumn{2}{|c|}{} & 1 \\
			\multicolumn{2}{|c|}{} & 1 \\
			\multicolumn{2}{|c|}{} & 1 \\
			\multicolumn{2}{|c|}{} & 1 \\
			\hline
		\end{tabular}
		\caption{Forward pass operator}
		\label{fig::backward}
	\end{subfigure}
	\vfill
	\caption{(a) represents the typical deep learning \ac{MLP} equivalent as a \ac{VNN}, where inputs are given by $x_i$, the Codd states $s$ are hard-coded to ones for hidden layers, then activated through the non-linear function $\sigma$ when firing, and zeros for signals. (b) shows the forward pass Green's function that approximates the diffusion process. (c) shows the backward pass Green's function that approximates the diffusion process. Both (a) and (b) can be used to approximate the true forward pass and backpropagation numerically with convolutions only, the latter would be weighted by the learning rate. Note that (a) and (b) are flipped horizontally in deep learning frameworks, where convolutions are commonly implemented as correlations rather than strict convolutions for efficiency.}
	\label{fig::Correspondence}
\end{figure}
\begin{theorem}[Correspondence-based Universality] \label{thm::Correspondence}
	The \Neurocell~framework is equivalent to deep learning as a special case, where the Green's function and Codd states are fixed, and that this case is equivalently Turing complete.
\end{theorem}
\begin{proof}
	We can view also this from the traditional deep learning formulation point of view using the matrix-vector form of an \ac{MLP}. Consider the \ac{VNN} \ac{MLP} with the depth along the row dimension (left to right) within the Chua field $C$ that consists inputs as a vector of length 8, 1 hidden layer of length 4 and 1 output layer of length 2 (see Figure~\ref{fig::Correspondence}). Thus, we have a hyperplane $d=2$ form of \acp{VNN}, where each of the layers take the form of vectors of dimension $d-1$. We could place the input vector in 8 adjacent cells along one column as a column vector with all these 8 cells having decision Codd states $s=1$. Then a column for signal neurons to hold the connections between inputs as the connection weights with $s=0$ in-between the input layer and the subsequent hidden layer of 4 cells of decision Codd states $s=1$. Likewise, this is followed by a column for signal neurons creating another signal layer and then the output layer cell. The forward pass is then given by the element-wise multiplication of the corresponding column from the tensor of weights $\bm{W}$ with the resulting convolution of the Green's function consisting of all ones with the input and the signal layers (see Figure~\ref{fig::forward} and~\ref{fig::backward}). 
	Because of this correspondence with deep learning, whose most general architectures such as the transformer are known to be Turing complete~\citep{perez_turing_2018}, the \Neurocell~ framework is also Turing complete because the multiplications and convolutions can be represented equivalently as a matrix that results in a toeplitz matrix (which is a circulant matrix if using circular convolutions and block toeplitz matrices for minibatches), and thus corresponds to the usual deep learning formulation using matrix-vector forms. In other words, the static hyperplane \acp{VNN} are a special case of traditional deep learning \acp{MLP}, where the matrix multiplications take the form of toeplitz matrices.
\end{proof}

\section{Data-driven Green's Functions}\label{sec::Greens}
In this section, we show the various Green's functions learned from the models covered in this work. Figures~\ref{fig::GreensALU} and~\ref{fig::GreensALU_TM} shows the 16 Green's functions found for the \ac{ALU} models presented in Figure~\ref{fig::Results_ALU} (middle and right).
\begin{figure}[ht!]
	\centering
	\includegraphics[width=\textwidth]{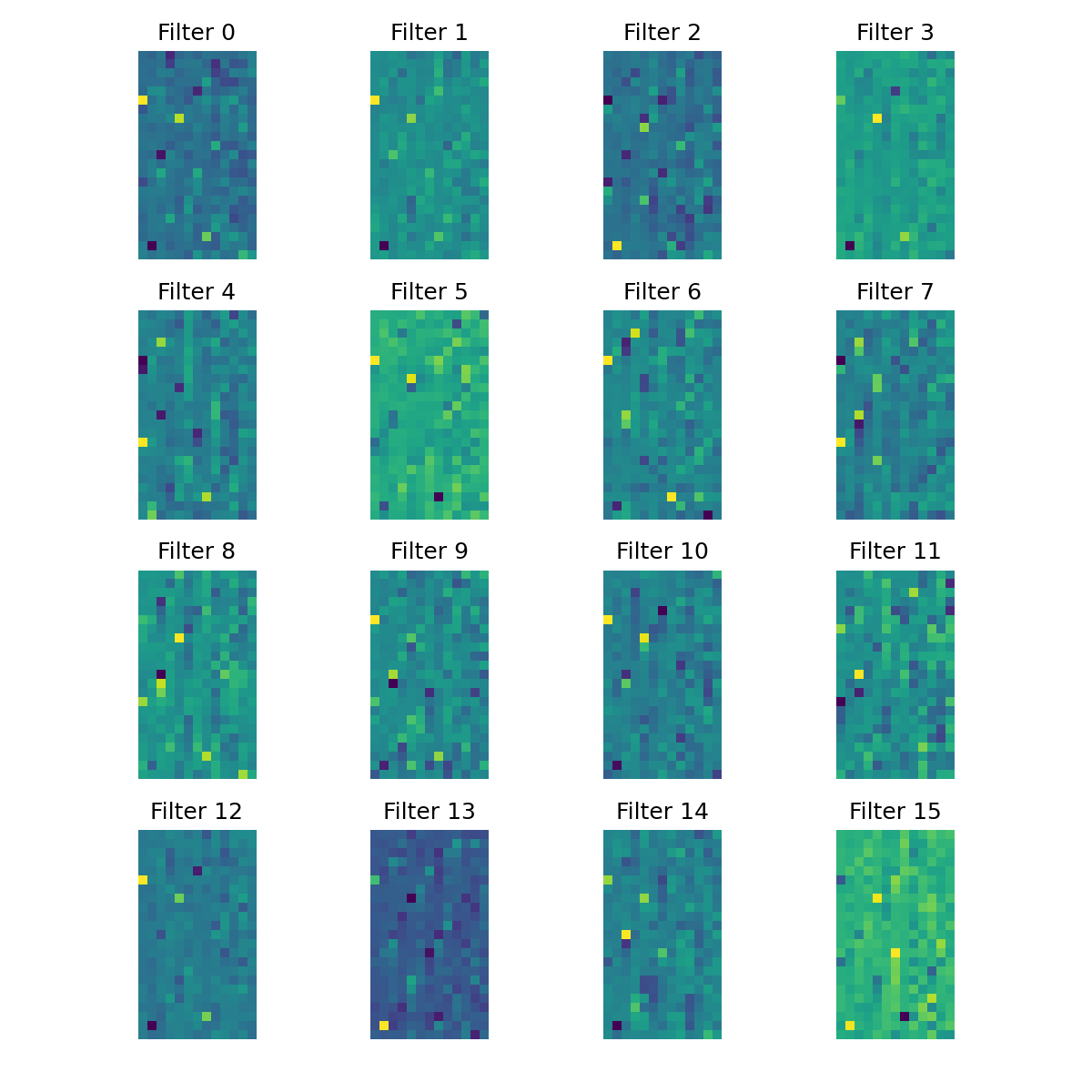}
	\caption{The resulting Green's functions learned for the box \ac{ALU} model from Figure~\ref{fig::Results_ALU} (middle).}
	\label{fig::GreensALU}
\end{figure}
\begin{figure}[ht!]
	\centering
	\includegraphics[width=\textwidth]{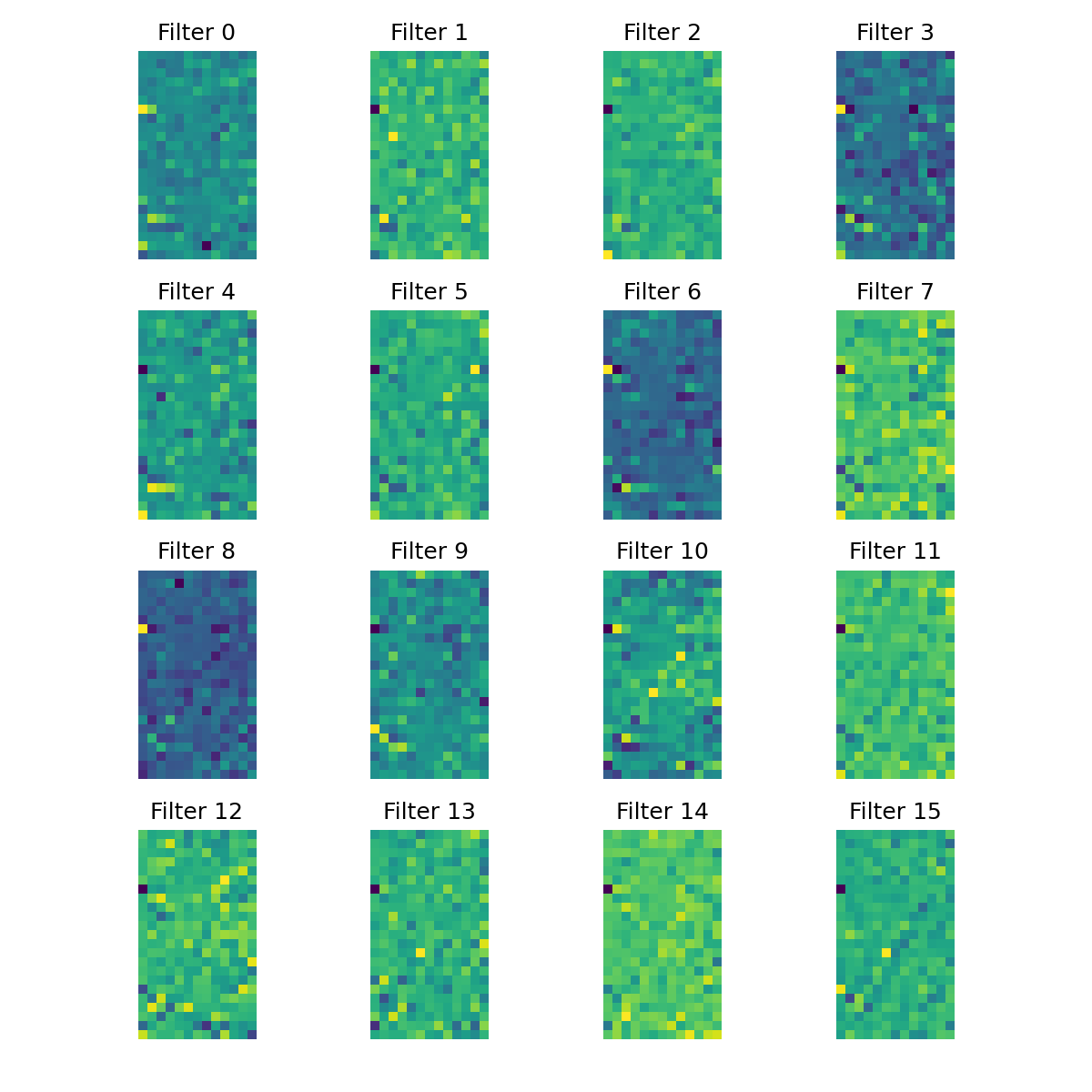}
	\caption{The resulting Green's functions learned for the tape-like \ac{ALU} model from Figure~\ref{fig::Results_ALU} (right).}
	\label{fig::GreensALU_TM}
\end{figure}
Figure~\ref{fig::CLP} shows the resulting 3 Green's functions from the \ac{CLP} model for simple classification of the Wine dataset~\citep{aeberhard_comparative_1994} provided within Scikit-Learn~\citep{scikit-learn} from Figure~\ref{fig::Results_ALU} (left). The functions show a similar symmetry to the original half-circular shape of the input.
\begin{figure}[ht!]
	\centering
	\includegraphics[width=0.5\textwidth]{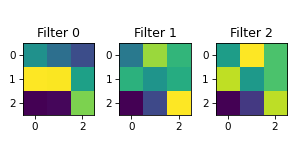}
	\caption{The resulting Green's functions learned for the \ac{CLP} model from Figure~\ref{fig::Results_ALU} (left).}
	\label{fig::CLP}
\end{figure}
Figure~\ref{fig::GoL_Filters} shows the learnt Green's functions of the \ac{GoL} (see Appendix~\ref{sec::GoL} for a brief introduction) experiments in simulating the \ac{GoL} with a \ac{VNN} (see Appendix~\ref{sec::gol-results} for details). The learnt 3x3 neighborhoods seem to encapsulate the \ac{GoL} rules despite being compact in number (i.e. less than the number required by $2^9$ possible combinations expected in the rules).
\begin{figure}[htp!]
	\centering
	\includegraphics[width=0.8\textwidth]{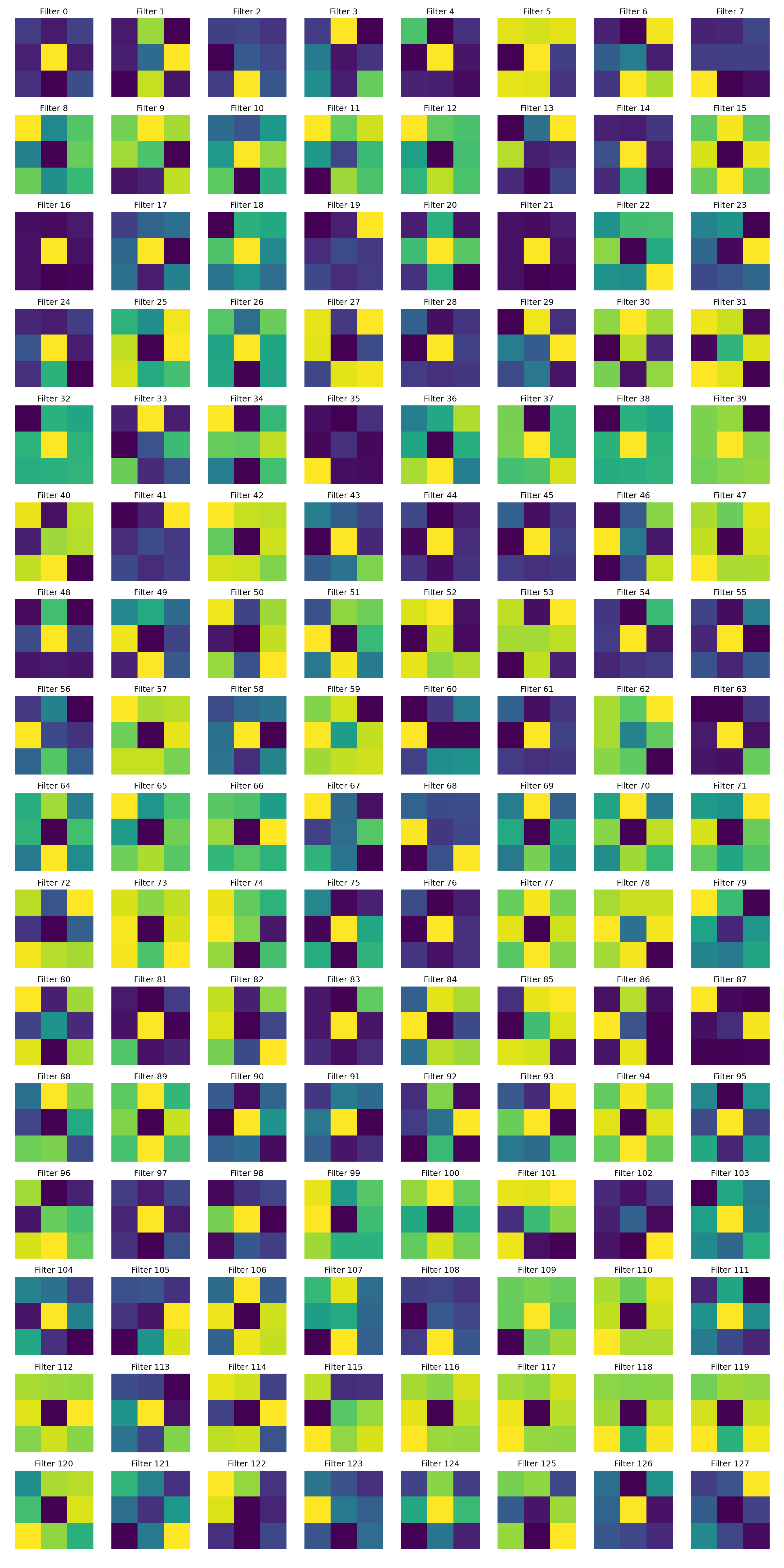}
	\caption{The resulting 128 data-driven Green's functions learnt for the \ac{GoL} \ac{VNN} simulation.}
	\label{fig::GoL_Filters}
\end{figure}

\section{Computation}\label{sec::TM}
[The following are excerpts from notes developed by the authors for a course taught on a similar topic]

Imagine that you could reduce all known computations, anything you know a computer can do and even arithmetic, down to just a few fundamental operations of reading, writing and moving symbols around on pieces of paper and nothing more. It would be simple enough to even solve any computation by hand given enough patience. With such a simple machine, we could not only study computation and prove its properties, but also help design a physical realisation of such a computation machine. Alan Turing~\citep{turing_computable_1937} introduced just such a machine in order to study the ``Entscheidungsproblem'' (or the decision problem) and in the process developed a theory of computation. These machines are known as \acfp{TM} in his honour and eventually led to the design the first digital computers by John von Neumann and others.

In the remaining sections of this chapter, we will construct these machines and design programs for them. In order to do this, we first investigate simpler machines that are possible, which do not compute everything, but are still useful for very simple specialised devices. These machines are devices that do not have any memory.

\subsection{Finite State Machines}
Consider having to design a device or system that only needs to decide whether to accept an input or reject it. Equivalently, you would like a system that produces a yes/on or no/off decision. In such a scenario, a machine that has a finite number of states and one that processes the input one element at a time is appropriate. This type of device that has no memory (in the sense of a tape or array) and is called a \acf{FSM}. First, we need to define the symbols that we will use for the input and the computation to apply on. The set of all possible unique symbols or characters is called the alphabet $\set{A}$. The input is then a string of symbols concatenated together from the alphabet. 

Formally, each \ac{FSM} firstly consists of pre-defined set of states $\set{S}=\{s_1,s_2,\ldots,s_n\}$ for number of states $n$, which is used during the computation. This includes three important states that all \acp{FSM} have: an initial state $s_0$, an accept state $s_{\textrm{yes}}$ and a reject state $s_{\textrm{no}}$. Secondly, we need the set of transitions or instructions $\delta$ for the \ac{FSM} that defines what state the machine should go into when in a particular state and it reads a certain character from the input string. In notation, we mean that $\delta : \set{S}\times\set{A}\to \set{S}$, i.e. state paired with character to another state.

\subsection{Turing Machines}
We saw that \acp{FSM} have an alphabet $\set{A}$, associated states $\set{S}$ and transition rules $\delta:\set{S}\times\set{A} \to \set{S}$. It does not however, have memory for working as a human would need to do arbitrary computations. This would also require keeping track of the ``awareness'' within that memory. Thus, we define a more powerful machine that has a tape for memory and a read/write head for keeping track of awareness, which is known as a \ac{TM}. The tape is a \ac{1D} array that can be thought of as sequential access memory. The read/write head moves left or right along the array and points to a cell in the tape that corresponds to the current computation taking place. The tape may hold the input and output strings, any intermediate results or working and even the program as we shall see later in this section. Figure~\ref{fig::TuringMachine} shows an illustration of the tape and head. 
\begin{figure}[htb]
	\centering
	\includegraphics[width=0.3\textwidth]{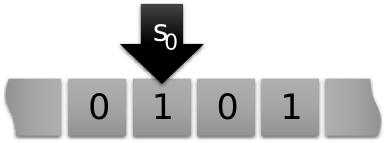}
	\caption{An illustration of a Turing machine with its tape and read/write head. In theory, the tape extends to infinity on both sides.}
	\label{fig::TuringMachine}
\end{figure} 

For theoretical reasons, we may consider the tape to be infinite in length. Unlike \acp{FSM}, the head also permits us to back along the input and mark on the tape. In fact, we are permitted to write any symbol we like as long as we define a tape alphabet $\set{A}_T$ that is a superset of symbols permitted on the tape with the input alphabet $\set{A}_i$. 

With memory available for working and ability to read/write and move the head, we have more degrees of freedom with our transitions. Indeed, now our transitions have the form $\delta: \set{S}\times\set{A}_T \to \set{S}\times\set{A}_T\times\{L,R\}$, where $L$ and $R$ represent the left and right movement of the head. Finally, the Turing machine also has an initial state $s_0$ like an \ac{FSM}, but only has single accept and reject states, $s_a$ and $s_r$ respectively, because we have the ability to complete the computation instantly without having to process the entire input string since we are free to move back and forth along the input. Thus, we can formally define the concept of a Turing machine. 
\begin{definition}[Turing Machine]\label{def::TuringMachine}
	A \acl{TM} is a seven tuple $T=(\set{S},\set{A}_i,\set{A}_T,\delta,s_0, s_a, s_r)$, where the sets $\set{S},\set{A}_i$ and $\set{A}_T$ are all finite and
	\begin{enumerate}
		\item $\set{S}$ is the set of all possible states,
		\item $\set{A}_i$ is the input alphabet without the blank symbol,
		\item $\set{A}_T$ is the tape alphabet which includes the blank symbol,
		\item $\delta: \set{S}\times\set{A}_T \to \set{S}\times\set{A}_T\times\{L,R\}$ is the transition function,
		\item $s_0\in \set{S}$ is the initial state,
		\item $s_a\in \set{S}$ is the accept state,
		\item $s_r\in \set{S}$ is the reject state with $s_r \neq s_a$.
	\end{enumerate}
\end{definition}

The Turing machine would function as the following:
\begin{enumerate}
	\item We assume that the tape is infinite in both directions.
	\item Begin with the input placed in the middle of the tape with the head in the left most cell of the input. The rest of the tape is considered empty or blank. 
	\item The input alphabet $\set{A}_i$ does not have the blank symbol, so it is safe to assume that the end of the input is marked with a blank symbol.
	\item The computation begins by using the tape alphabet $\set{A}_T$ and following the transition function until either the accept or reject state is reached. 
	\item When either of $s_a$ or $s_r$ is reached, the machine halts. Otherwise, the machine continues forever without halting.
\end{enumerate}
A detailed exposition of computation theory, including descriptions of both classical and quantum computing can be found in~\citet{moore_nature_2011}.

\subsection{Equivalences}
Firstly, if we add higher dimensional tape, such as \ac{2D} or \ac{3D} tape, to increase the versatility of our memory. It is easy to show that \ac{2D} or \ac{3D} arrays are equivalent to \ac{1D} arrays. and in fact, computers have been using \ac{1D} arrays to represent multi-dimensional arrays very early on in the history of computers. This is because we can always use memory addressing tricks to simulate multi-dimensional arrays using \ac{1D} coordinates. 

Secondly, we could use multiple heads to improve the speed and support more instructions across the tape and transition function. However, this simply means that we can process instructions faster, but no more improvement of the computational power of the machine. In a similar argument to simulating multi-dimensional arrays for memory with a \ac{1D} array, we can do the same with a single read/write head. The single head can complete each of the operations required for other heads one after the other sequential until all the operations of all heads are complete.

\section{Game of Life}\label{sec::GoL}
The \acf{GoL} was introduced by John Conway~\citep{gardner_mathematical_1970} as the culmination of cellular automata research in which the rules have been reduced to a few simple rules. The game involves a set of cells within an $N\times N$ grid whose state is either alive or dead (i.e. 1 or 0 respectively). The grid effectively represents the ‘universe’ that will be simulated and the alive cells the life within it. The game is governed by a set of simple rules that dictate the next state of each cell in this universe depending on its current state and the states of its neighbours. The rules of \ac{GoL} depend on the 8-connected neighbours of a cell as follows:
\begin{enumerate}
	\item Underpopulation: A live cell that has < 2 live neighbouring cells will die.
	\item Survival: A live cell that has 2-3 live neighbouring cells will remain alive.
	\item Overpopulation: A live cell with more than 3 live neighbours will die.
	\item Reproduction: A dead cell with exactly 3 live neighbours will become alive.
\end{enumerate}
The game begins with an initial state of the universe with a pattern of live cells. The universe is evolved by applying the above rules to each cell of the universe to determine the next iteration of the simulation. The evolution of the universe is observed by continual computing of the next iteration of the universe. 

\section{Abbreviations}
\begin{acronym}[GCD]
	\acro{1D}{one dimensional}
	\acro{2D}{two dimensional}
	\acro{3D}{three dimensional}
	\acro{PDE}{partial differential equation}
	\acro{PSF}{point spread function}
	\acro{DFT}{discrete Fourier transform}
	\acro{CeNN}{Cellular Neural Network}
	\acro{CNN}{Convolutional Neural Network}
	\acro{VNN}{Von Neumann Network}
	\acro{GCD}{greatest common divisor}
	\acro{ACE}{automatic computing engine}
	\acro{FSM}{finite state machine}
	\acro{TM}{Turing machine}
	\acro{UTM}{universal Turing machine}
	\acro{CM}{cellular machine}
	\acro{UCM}{universal cellular machine}
	\acro{CPU}{central processing unit}
	\acro{GPU}{graphical processing unit}
	\acro{ALU}{arithmetic logic unit}
	\acro{EoC}{Everything on Chip}
	\acro{SoC}{System on Chip}
	\acro{CA}{cellular automata}
	\acro{GoL}{Game of Life}
	\acro{QM}{quantum mechanics}
	\acro{MLP}{multi-layered perceptron}
	\acro{CLP}{circular layered perceptron}
	\acro{NAS}{neural architecture search}
	\acro{MMA}{massive multi-agent}
\end{acronym}


\end{document}